\documentclass[journal]{IEEEtran}

\usepackage{cite}
\usepackage{graphicx}
\usepackage{textcomp}

\usepackage{amsmath,amsfonts,amssymb,color}
\usepackage{amsthm,amsmath,amsfonts,amssymb,color}
\usepackage{graphicx}
\usepackage{psfrag}
\usepackage{algorithm}
\usepackage{algpseudocode}
\usepackage{array}
\usepackage{cite}
\usepackage{footmisc}
\usepackage{mathtools}
\usepackage{enumerate}
\usepackage{soul,color}
\usepackage{caption}
\usepackage{subfig}

\usepackage{nccmath}

\newtheorem{theorem}{Theorem}
\newtheorem{coro}{Corollary}
\newtheorem{lemma}{Lemma}
\newtheorem{remark}{Remark}
\newtheorem{definition}{Definition}
\newtheorem{proposition}{Proposition}

\newtheorem{assumption}{Assumption}

\renewcommand{\theenumi}{(\alph{enumi}}

\DeclareMathOperator{\diag}{diag}  
\DeclareMathOperator{\tr}{tr}
\DeclarePairedDelimiter\floor{\lfloor}{\rfloor}

\begin{document}
\title{Learning Dynamical Systems by Leveraging Data from Similar Systems}
\author{~Lei~Xin, Lintao Ye, George Chiu, Shreyas Sundaram  
\thanks{TThis research was supported by USDA grant 2018-67007-28439.  This work represents the opinions of the authors and not the USDA or NIFA. Lei Xin and Shreyas Sundaram are with the Elmore Family School of Electrical and Computer Engineering, Purdue University. George Chiu is with the School of Mechanical Engineering, Purdue University. E-mails: {\tt\{lxin, gchiu, sundara2\}@purdue.edu}. Lintao Ye is with the School of Artificial Intelligence and Automation, Huazhong University of Science and Technology. E-mail: {\tt yelintao93@hust.edu.cn}.}}

\maketitle

\begin{abstract}
We consider the problem of learning the dynamics of a linear system when one has access to data generated by an auxiliary system that shares similar (but not identical) dynamics, in addition to data from the true system. We use a weighted least squares approach, and provide finite sample error bounds of the learned model as a function of the number of samples and various system parameters from the two systems as well as the weight assigned to the auxiliary data. We show that the auxiliary data can help to reduce the intrinsic system identification error due to noise, at the price of adding a portion of error that is due to the differences between the two system models. We further provide a data-dependent bound that is computable when some prior knowledge about the systems, such as upper bounds on noise levels and model difference, is available. This bound can also be used to determine the weight that should be assigned to the auxiliary data during the model training stage. 
\end{abstract}

\begin{IEEEkeywords}
System identification, machine learning, sample complexity
\end{IEEEkeywords}

\section{Introduction} 

\label{sec: introduction}
Building an accurate predictive model for a dynamical system is crucial in various fields, including control theory, reinforcement learning, and economics \cite{ljung1999system}. The problem of system identification seeks to learn the system model from data when modeling from first principles is not possible. While classical system identification techniques and their associated theories focused primarily on achieving asymptotic consistency \cite{bauer1999consistency,jansson1998consistency,knudsen2001consistency}, recent efforts have sought to characterize the performance of the learned model given a finite number of samples, inspired by advances in machine learning. Having a finite sample bound for the estimation error is not only of interest on its own, but also can be integrated with techniques like robust control to come up with overall performance guarantees for the closed loop system, e.g., \cite{dean2018regret,dean2019sample,ye2021sample}.

The existing literature on finite sample analysis of system identification is typically either single trajectory-based or multiple trajectories-based. The single trajectory setup assumes that one has samples from a single trajectory from the system, which enables system identification in an online manner, i.e., there is no need to restart the system from an initial state.  This setup has been studied extensively over the past few years and is still an ongoing research topic \cite{simchowitz2018learning,oymak2019non,simchowitz2019learning,sarkar2021finite,faradonbeh2018finite,sarkar2019near}. A key challenge in the analysis is addressing the dependencies of samples from the single trajectory. The derived sample complexity results typically show how the system identification error goes to zero by increasing the number of samples used in the single trajectory. For the multiple trajectories setup, it is typically assumed that one has access to data generated from multiple independent trajectories \cite{dean2019sample, fattahi2018data, sun2020finite, zheng2020non, xin2022learning}. In practice, the multiple trajectories setup has the advantage of being able to handle unstable systems, and other cases where collecting a single long trajectory is infeasible. Technically, the assumption of independence of data 
usually allows for more direct use of standard concentration inequalities. Consequently, the derived results typically only show that the error goes to zero by increasing the number of trajectories. The recent paper \cite{tu2022learning} carefully addresses learning dynamical systems from a mix of both dependent and independent data, i.e., learning from multiple trajectories each with a non-trivial length. The paper \cite{tu2022learning} provides sharp bounds that hold in expectation, and shows that the error goes to zero at a rate that is determined by the product of the number of trajectories and the number of samples used from these trajectories. 

We note that all of the above works make the assumption that the data used for system identification are generated from the true system model that one wants to learn. However, in many cases, collecting abundant data from the true system can be costly or even infeasible. In such cases, one may want to rely on data generated from other systems that share similar dynamics. For example, for non-engineered systems like animals, one may only have a limited amount of data from the true individual animal one wants to model, due to the challenge of conducting experiments for such systems. On the other hand, it may be possible to collect data from other animals in the herd or from a reasonably good simulator. Furthermore, when the dynamics of a system changes (e.g., due to failures), one needs to decide whether to discard all of the previous data, or to leverage the old information in an appropriate way. In settings such as the ones described above, it is of great interest to understand how one can leverage the data generated from systems that share similar (but not identical) dynamics. This idea is similar to the notion of {\it transfer learning} in the machine learning community, where one wants to transfer knowledge from related tasks to a new task \cite{pan2009survey}. However, in contrast to system identification, most of the papers on transfer learning (in the context of estimation) consider learning a static mapping from a feature space to a label space \cite{bastani2021predicting}. The recent works \cite{modi2021joint,faradonbeh2022joint} study joint learning of multiple dynamical systems, assuming all systems are weighted equally in the training stage. However, an open question remains on how to effectively utilize samples from other systems to enhance the accuracy of the model for the true system of interest, especially when the number of samples from the true system is limited.

Our conference paper \cite{xin2022identifying} provides finite sample analysis of system identification with the help of an auxiliary system, using a weighted least squares approach, under the assumption of having access to multiple trajectories from both the true system and the auxiliary system. The paper \cite{xin2022identifying} decomposes the overall system identification error into the error due to noise and the error due to model difference, and shows that the auxiliary data can help to reduce the error due to noise by introducing a portion of constant error that is due to the difference in the models between the true and auxiliary systems. However, although the algorithm in \cite{xin2022identifying} uses all samples from these trajectories (different from \cite{dean2019sample}, where only two data points from each trajectory are used, assuming all samples are generated from the same system), the result is conservative in characterizing the effect of the trajectory length. In particular, the error due to noise can only go to zero by increasing the number of trajectories from the two systems. 


In this paper, we address the above problem. Our contributions are as follows. 

\begin{itemize}
  \item We provide finite sample data-independent bounds for learning dynamical systems by leveraging data from an auxiliary system, using a weighted least squares approach. Again, we decompose the error into a portion due to noise and a portion due to model difference. Different from \cite{xin2022identifying}, we show that the error due to noise can go to zero by increasing either the number of trajectories or the trajectory length from the two systems, or both. Our analysis is general in that when the two systems have same system matrices (such that we only have the error due to noise), our result qualitatively matches the results in the recent paper \cite{tu2022learning}, which characterizes how the expected error goes to zero with respect to the number of trajectories and the trajectory length, given samples from the same system. 
 Importantly, our bounds provide insights on general guidelines for assigning weights to the auxiliary system, when there is not enough prior knowledge about the systems.

  \item We also provide a data-dependent bound that is computable when some prior knowledge about the systems, such as upper bounds on noise levels and model difference, is available (based on a regularized weighted least squares approach). The data-dependent bound can be used in a data-driven scheme for selecting a good weight parameter that provides better performance guarantees in practice.
  
\end{itemize}

To the best of the authors' knowledge, our paper is the first to study finite sample analysis for weighted least squares-based system identification given different systems. Technically, we overcome the challenges of addressing the dependencies of samples from independent system trajectories in a less conservative way, compared to \cite{xin2022identifying, dean2019sample}. 
We also provide a new lower bound for the smallest eigenvalue of the sample covariance matrix for non-Gaussian time series in a more general context. This result could be of independent interest since it can be used in the analyses of many regression-based problems.

Our paper is organized as follows. Section \ref{sec: notation and terminology} introduces relevant mathematical notation and terminology. Section \ref{sec: problem formulation} formulates the system identification problem and introduces the algorithm we use. In Section \ref{Analysis}, we present our main results. We present various numerical examples capturing different scenarios in Section \ref{exp} to illustrate our results, and conclude in Section \ref{conclusion}. All of the proofs are included in the appendix.

\section{Mathematical Notation and Terminology} 

\label{sec: notation and terminology}
Let $\mathbb{R}$ denote the set of real numbers. The symbol $\cup$ is used to denote the union of sets. Let $\lambda_{min}(\cdot)$ and $\lambda_{max}(\cdot)$ be the smallest and largest eigenvalues, respectively, of a symmetric matrix. The spectral radius of a given matrix is denoted as $\rho(\cdot)$. A square matrix $A$ is called strictly stable if $\rho(A)<1$, marginally stable if $\rho(A)\leq1$, and unstable if $\rho(A)>1$. We use $\|\cdot\|$ and $\|\cdot\|_{F}$ to denote the spectral norm and Frobenius norm of a given matrix, respectively. Vectors are treated as column vectors, and the symbol $'$ is used to denote the transpose operator. We use $\tr(\cdot)$ to denote the trace of a given matrix. We use $I_{n}$ to denote the identity matrix with dimension $n \times n$. The symbol $\sigma(\cdot)$ is used to denote the sigma field generated by the corresponding random vectors. We use $\mathcal{S}^{n-1}$ to denote the unit sphere in $n$-dimensional space. 



\section{Problem formulation and algorithm} 
\label{sec: problem formulation}
Consider the following discrete time linear time-invariant (LTI) system
\begin{equation}
\begin{aligned} 
\bar{x}_{t+1}=\bar{A}\bar{x}_{t}+\bar{B}\bar{u}_{t}+\bar{w}_{t}, \\
\end{aligned}
\label{eq:True system}
\end{equation}
where $\bar{x}_{t}\in \mathbb{R}^{n}$, $\bar{u}_{t}\in \mathbb{R}^{p}$, $\bar{w}_{t}\in \mathbb{R}^{n}$, are the state, input, and process noise, respectively, and $\bar{A}\in \mathbb{R}^{n \times n}$ and $\bar{B}\in \mathbb{R}^{n \times p}$ are the system matrices we wish to learn from data. 
In this paper, we also assume that both the input $\bar{u}_{t}$ and state $\bar{x}_{t}$ can be perfectly measured. 

Suppose that we have access to $N_{r}$ independent experiments of system \eqref{eq:True system}, in which the system restarts from an initial state $\bar{x}_{0}$, and each experiment is of length $T_{r}$. We call the state-input pairs collected from each experiment a {\it rollout} (or trajectory), and denote the set of samples we have as $\{(\bar{x}^{i}_{t},\bar{u}^{i}_{t}):1 \leq i \leq N_{r},0 \leq t \leq T_{r}\}$. Note that we use the superscript to denote the rollout index and the subscript to denote the time index.

Let $\bar{z}^{i}_{t}=\begin{bmatrix} \bar{x}^{i'}_{t}& \bar{u}^{i'}_{t} \end{bmatrix}^{'}\in \mathbb{R}^{n+p}$ for $t\geq 0$. For each rollout $i$, define $\bar{X}^{i}=\begin{bmatrix}
\bar{x}^{i}_{1}&\cdots&\bar{x}^{i}_{T_{r}}
\end{bmatrix}  \in \mathbb{R}^{n\times T_{r}}$, $\bar{Z}^{i}=\begin{bmatrix}
\bar{z}^{i}_{0}&\cdots&\bar{z}^{i}_{T_{r}-1}
\end{bmatrix} \in \mathbb{R}^{(n+p)\times T_{r}}$, $\bar{W}^{i}=\begin{bmatrix}
\bar{w}^{i}_{0}&\cdots&\bar{w}^{i}_{T_{r}-1}
\end{bmatrix}  \in \mathbb{R}^{n\times T_{r}}$.
Further, define the batch matrices $\bar{X}=\begin{bmatrix}\bar{X}^{1}&\cdots&\bar{X}^{N_{r}}\end{bmatrix}\in \mathbb{R}^{n\times N_{r}T_{r}},\bar{Z}=\begin{bmatrix}\bar{Z}^{1}&\cdots&\bar{Z}^{N_{r}}\end{bmatrix}\in \mathbb{R}^{(n+p)\times N_{r}T_{r}},\bar{W}=\begin{bmatrix}\bar{W}^{1}&\cdots&\bar{W}^{N_{r}}\end{bmatrix}\in \mathbb{R}^{n\times N_{r}T_{r}}$. Denoting $\Theta \triangleq \begin{bmatrix}
\bar{A}&\bar{B}\end{bmatrix}\in \mathbb{R}^{n\times(n+p)}$, we have
\begin{equation*}
\begin{aligned}
&\bar{X}=\Theta\bar{Z}+\bar{W}.
\end{aligned}
\end{equation*}
Typically, one would like to solve the following optimization problem:
\begin{equation*} 
\begin{aligned}
    \mathop{\min}_{\tilde{\Theta}\in \mathbb{R}^{n\times (n+p)}} \|\bar{X}-\tilde{\Theta}\bar{Z}\|^{2}_{F},
\end{aligned}
\end{equation*}
and obtain an estimate $\Theta_{LS} \triangleq \begin{bmatrix}\bar{A}_{LS}& \bar{B}_{LS}\end{bmatrix}$, of which the analytical form is
\begin{equation*} 
\begin{aligned}
\Theta_{LS}=\bar{X}\bar{Z}^{'}(\bar{Z}\bar{Z}^{'})^{-1},
\end{aligned}
\end{equation*}
under the assumption that the matrix $\bar{Z}\bar{Z}^{'}$ is invertible. The quality of the recovered estimate will depend on $N_{r}$ and $T_{r}$; in particular, if both $N_{r}$ and $T_{r}$ are small, the obtained estimate could have large estimation error \cite{dean2019sample, simchowitz2018learning}.

Suppose that, in addition to samples from the true system, we also have access to samples generated from an auxiliary system that shares ``similar'' (but unknown) dynamics to system \eqref{eq:True system}. In particular, consider an auxiliary discrete time linear time-invariant system 
\begin{equation}
\begin{aligned} 
\hat{x}_{t+1}=\hat{A}\hat{x}_{t}+\hat{B}\hat{u}_{t}+\hat{w}_{t},\\
\end{aligned}
\label{eq:Perturbed system}
\end{equation}
where $\hat{x}_{t}\in \mathbb{R}^{n}$, $\hat{u}_{t}\in \mathbb{R}^{p}$, $\hat{w}_{t}\in \mathbb{R}^{n}$ are the state, input, and process noise, respectively, and $\hat{A}\in \mathbb{R}^{n \times n}$ and $\hat{B}\in \mathbb{R}^{n \times p}$ are system matrices. The above dynamics can be rewritten as
\begin{equation} 
\begin{aligned}
\hat{x}_{t+1}=(\bar{A}+\delta_{A})\hat{x}_{t}+(\bar{B}+\delta_{B})\hat{u}_{t}+\hat{w}_{t},\\
\end{aligned}
\end{equation}
where $\delta_{A}=\hat{A}-\bar{A},\delta_{B}=\hat{B}-\bar{B}$. Intuitively, the samples generated from the above system will be useful for identifying system \eqref{eq:True system} if both $\|\delta_{A}\|$ and $\|\delta_{B}\|$ are small. For example, if we want to identify the dynamics of a vehicle, the auxiliary system could be another vehicle of the same type produced by the same manufacturer. We also provide a scenario involving a change in dynamics of a system where the true and the auxiliary systems have the same state representation in our experiment section later in the paper.\footnote{In practice, the auxiliary system needs to share the same set of state variables to be considered ``similar'', although our results hold for general systems where the states are unrelated.} 

Thus, suppose that we also have access to $N_{p}$ independent experiments of system \eqref{eq:Perturbed system}, in which the system restarts from an initial state $\hat{x}_{0}$, and each experiment is of length $T_{p}$. Let $\{(\hat{x}^{i}_{t},\hat{u}^{i}_{t}):1 \leq i \leq N_{p},0 \leq t \leq T_{p}\}$ denote the samples from these experiments. Let $\hat{z}^{i}_{t}=\begin{bmatrix} \hat{x}^{i'}_{t}& \hat{u}^{i'}_{t} \end{bmatrix}^{'}\in \mathbb{R}^{n+p}$ for $t\geq 0$. The matrices  $\hat{X}^{i}\in \mathbb{R}^{n\times T_{p}},\hat{Z}^{i}\in \mathbb{R}^{(n+p)\times T_{p}},\hat{W}^{i}\in \mathbb{R}^{n\times T_{p}},\hat{X}\in \mathbb{R}^{n\times N_{p}T_{p}},\hat{Z}\in \mathbb{R}^{(n+p)\times N_{p}T_{p}},\hat{W}\in \mathbb{R}^{n\times N_{p}T_{p}}$ are defined similarly, using the signals $\hat{u}^{i}_{t},\hat{x}^{i}_{t},\hat{w}^{i}_{t}$ from system \eqref{eq:Perturbed system}.  Let $X=\begin{bmatrix}\bar{X}&\hat{X}\end{bmatrix}\in \mathbb{R}^{n\times (N_{r}T_{r}+N_{p}T_{p})}, Z=\begin{bmatrix}\bar{Z}&\hat{Z}\end{bmatrix}\in \mathbb{R}^{(n+p)\times (N_{r}T_{r}+N_{p}T_{p})}, W=\begin{bmatrix}\bar{W}&\hat{W}\end{bmatrix}\in \mathbb{R}^{n\times (N_{r}T_{r}+N_{p}T_{p})}$ and $\delta_{\Theta}=\begin{bmatrix}
\delta_{A}&\delta_{B}\end{bmatrix}\in \mathbb{R}^{n\times (n+p)}$. Defining 
\begin{equation*}
\Delta^{i}=\begin{bmatrix}
\delta_{\Theta} \hat{z}^{i}_{0}&\cdots&\delta_{\Theta}\hat{z}^{i}_{T_{p}-1}
\end{bmatrix}\in \mathbb{R}^{n\times T_{p}},\\
\end{equation*}
for all $i\in\{1,\dots,N_p\}$, and denoting
\begin{equation*}
\Delta=\begin{bmatrix}
\mathbf{0}&\cdots&\mathbf{0}&\Delta^{1} &\cdots& \Delta^{N_{p}}\end{bmatrix}\in \mathbb{R}^{n\times (N_{r}T_{r}+N_{p}T_{p})},
\end{equation*}
where we use $\mathbf{0}$ to denote zero matrices with appropriate dimensions, we have the relationship
\begin{equation} \label{relationship}
\begin{aligned}
&X=\Theta Z+W+\Delta.
\end{aligned}
\end{equation}

Letting $q \in \mathbb{R}_{\ge 0}$ be a design parameter that specifies the relative weight assigned to samples generated from the auxiliary system \eqref{eq:Perturbed system}, we can define $Q=\diag(I_{N_{r}T_{r}},qI_{N_{p}T_{p}})\in \mathbb{R}^{(N_{r}T_{r}+N_{p}T_{p})\times (N_{r}T_{r}+N_{p}T_{p})}$. Setting the regularization parameter $\lambda\geq 0$, we are interested in the following (regularized-) weighted least squares problem:
\begin{equation} \label{Regularized weighted pb}
\begin{aligned}
    \mathop{\min}_{\tilde{\Theta}\in \mathbb{R}^{n\times (n+p)}} \{\|XQ^{\frac{1}{2}}-\tilde{\Theta}ZQ^{\frac{1}{2}}\|^{2}_{F}+\lambda \|\tilde{\Theta}\|^2_{F}\}.
\end{aligned}
\end{equation}
The well known (regularized-) weighted least squares estimate \cite{hoerl1970ridge} is $\Theta_{WLS} \triangleq\begin{bmatrix}\bar{A}_{WLS}&\bar{B}_{WLS}\end{bmatrix}$, which has the form
\begin{equation} 
\begin{aligned}
\Theta_{WLS}=XQZ'(ZQZ'+\lambda I_{n+p})^{-1},
\end{aligned}
\end{equation}
when the matrix $ZQZ'+\lambda I_{n+p}$ is invertible.
Using \eqref{relationship}, the system identification error can be expressed as
\begin{equation} 
\begin{aligned}
\|\Theta_{WLS}-\Theta\|&=\|-\lambda\Theta(ZQZ'+\lambda I_{n+p})^{-1}\\
&+WQZ'(ZQZ'+\lambda I_{n+p})^{-1}\\
&+\Delta QZ'(ZQZ'+\lambda I_{n+p})^{-1}\| .\label{error_W}
\end{aligned}
\end{equation}
In particular, when the regularization parameter is set to be $\lambda=0$, we recover the standard weighted least squares estimate.

\begin{remark}
The weight parameter $q$ specifies how much we weight the data from the auxiliary system relative to the data from the true system, and can depend on the number of samples ($N_{r},T_r$ and $N_{p}, T_p$) from each of those systems or the data available to us. The specific choice of $q$ will be discussed in detail later as we present our main results. 
\end{remark}


Our results will leverage the following definition of sub-Gaussian random vectors.
\begin{definition}
A real-valued random variable $w$ is called sub-Gaussian with parameter $R^2$ if  we have
\begin{equation*}
\begin{aligned}
&\forall \alpha\in \mathbb{R}, \mathbb{E}[\exp(\alpha{w})]\leq \exp(\frac{\alpha^2 R^2}{2}).\\
\end{aligned}
\end{equation*}
A random vector $x\in \mathbb{R}^n$ is called $R^2$ sub-Gaussian if for all unit vectors $v\in \mathcal{S}^{n-1}$ the random variable $v'x$ is $R^2$ sub-Gaussian.
\end{definition}

We make the following assumption.

\begin{assumption} \label{assumption}
The random vectors $\bar{w}_{t},\bar{u}_{t},\bar{x}_{0}, \hat{w}_{t}, \hat{u}_{t},\hat{x}_{0}$ are independent sub-Gaussian with independent coordinates for all $t\geq 0$. Furthermore, they have positive definite covariance matrices and sub-Gaussian parameters $\sigma_{\bar{w}}^2, \sigma_{\bar{u}}^2,\sigma_{\bar{x}_{0}}^2, \sigma_{\hat{w}}^2, \sigma_{\hat{u}}^2,\sigma_{\hat{x}_{0}}^2$, respectively.
\end{assumption}

We note that independent random inputs are commonly used in the context of system identification to provide excitation of the system dynamics  \cite{dean2019sample,oymak2019non}. Studying the optimal input for system identification is an active area of research \cite{wagenmaker2020active}. Further, random initial states can be easily obtained from deterministic initial states that are equal to zero. A simple way to achieve this is to apply zero input and treat the first state as the initial state for each trajectory (which is random due to noise). Our results could also be generalized to include bounded deterministic initial states.

To ease the notation, we will define some useful quantities. We denote $\bar{\sigma}_{max}=\max(\sigma_{\bar{w}}, \sigma_{\bar{u}},\sigma_{\bar{x}_{0}})$. Letting $v(j)$ denote the $j$-th component of a vector $v$, we define 
\begin{equation*}
\begin{aligned}
&\bar{\sigma}^{2}_{min}\triangleq\inf(\{\mathbf{E}[\bar{w}(i)^2],\mathbf{E}[\bar{u}(j)^2],\mathbf{E}[\bar{x}_{0}(i)^2]\})>0,\\
&\bar{\sigma}_{*}\triangleq\sup\left(\Bigg\{\frac{\mathbf{E}[\bar{w}(i)^4]}{\mathbf{E}[\bar{w}(i)^2]^2},\frac{\mathbf{E}[\bar{u}(j)^4]}{\mathbf{E}[\bar{u}(j)^2]^2},\frac{\mathbf{E}[\bar{x}_{0}(i)^4]}{\mathbf{E}[\bar{x}_{0}(i)^2]^2}\Bigg\}\right),\\
\end{aligned}
\end{equation*}
for all $t\geq 0, 1\leq i\leq n, 1\leq j \leq p$, where we omitted the time index $t$ for the ease of exposition. Further, define the following matrix for $t\geq 0$:



\begin{equation} \label{def: Gt}
\begin{aligned}
\bar{G}_{t}\triangleq \sum_{i=0}^{t}\bar{A}^{i}\bar{A}^{i}{'}+\sum_{i=0}^{t-1}\bar{A}^{i}\bar{B}\bar{B}'\bar{A}^{i}{'}.\\
\end{aligned}
\end{equation}
The terms $\hat{\sigma}_{max}, \hat{\sigma}^{2}_{min}, \hat{\sigma}_{*}$ and $\hat{G}_{t}$ are defined similarly for the auxiliary system \eqref{eq:Perturbed system}.

In the next section, we provide data-independent bounds (assuming $\lambda=0$) and a data-dependent bound (assuming $\lambda> 0$) of the system identification error in \eqref{error_W}. We study the case when $\lambda=0$ in the data-independent bounds to highlight our key insights (the benefits of the auxiliary data and the role of the weight parameter $q$), and the results for $\lambda>0$ can be easily generalized. The data-independent finite sample upper bounds characterize the error as a function of $N_{r}, T_{r}, N_{p}, T_{p}, q, \|\delta_{\Theta}\|$ and other parameters from the true system and the auxiliary system. While the data-independent error bounds provide insights on the benefits of using the auxiliary samples, the derived results may not be used directly in practice, since they involve unknown system parameters. To address that, we also provide a data-dependent bound for the case when $\lambda>0$. The non-zero regularization parameter $\lambda$ not only helps us to derive the data-dependent result, but also provides the user with more flexibility to tune the estimate in practice. One could set $\lambda$ to be small to reduce the impact of regularization on the estimate. The derived data-dependent bound is computable, applicable to more general input and noise, and can be used in real-world applications to select the weight parameter $q$ (and regularization parameter $\lambda$). More specifically, the bound characterizes the error as a function of $\sigma_{\bar{w}},\sigma_{\hat{w}}, q, \|\delta_{\Theta}\|$, $\lambda$, $\|\Theta\|$, and the available data. Both our data-independent bounds and data-dependent bound will provide insights and guidance on selecting an appropriate weight parameter $q$. We will assume that system \eqref{eq:True system} and system \eqref{eq:Perturbed system} have the same stability in our discussions, i.e., both $\rho(\bar{A})$ and $\rho(\hat{A})$ are less than 1, or both $\rho(\bar{A})$ and $\rho(\hat{A})$ are equal to 1, or both $\rho(\bar{A})$ and $\rho(\hat{A})$ are greater than 1 (although $\rho(\bar{A})$ does not need to equal to $\rho(\hat{A})$), but similar insights can be extended even if they are different.


\section{Analysis of the System Identification Error}
\label{Analysis}
In this section, we provide data-independent bounds (assuming $\lambda=0$), and a data-dependent bound (assuming $\lambda> 0$) on the system identification error in \eqref{error_W}. The proof of the data-independent bounds follow by upper bounding the error terms $\|WQZ'(ZQZ')^{-\frac{1}{2}}\|$, $\|(ZQZ')^{-\frac{1}{2}}\|$, and $\|\Delta QZ'\|$ separately. The proof of the data-dependent bound follows by directly evaluating an upper bound of the term \eqref{error_W} from data, but with the replacement of the noise-dependent term  $\|WQZ'(ZQZ'+\lambda I_{n+p})^{-\frac{1}{2}}\|$ by a high-probability upper bound. All of the proofs are presented in the appendix.

\subsection{Data-independent Bounds}
Here, we present our first main result, a data-independent finite sample upper bound  on the weighted least squares estimation error in \eqref{error_W} when $\lambda=0$. In the following result, we let $c,c_{1}$ denote some positive constants.\footnote{See Remark \ref{c} and Remark \ref{c1} in the appendix for more discussions on the constants $c,c_{1}$. }
\begin{theorem}
\label{data-independent bound}
Fix $q\geq 0$, $\delta \in (0,\frac{2}{e})$, and let Assumption \ref{assumption} hold. Denote $\bar{\zeta}=\frac{\bar{\sigma}_{min}}{c_{1}\bar{\sigma}_{*}}$ and $\hat{\zeta}=\frac{\hat{\sigma}_{min}}{c_{1}\hat{\sigma}_{*}}$. Suppose that $N_{r}T_{r}\geq \max\{33,8c_{1}^2\bar{\sigma}_{*}^2(\log\frac{2}{\delta}+(n+p)\log\frac{144\bar{g}(\frac{\delta}{2})}{\bar{\zeta}^2(N_{r}T_{r}-1)})+1\}$, $N_{p}T_{p}\geq \max\{33,8c_{1}^2\hat{\sigma}_{*}^2(\log\frac{2}{\delta}+(n+p)\log\frac{144\hat{g}(\frac{\delta}{2})}{\hat{\zeta}^2(N_{p}T_{p}-1)})+1\}$, $\bar{g}(\frac{\delta}{2}) \geq \frac{\bar{\zeta}^2(N_{r}T_{r}-1)}{16}$, and $\hat{g}(\frac{\delta}{2}) \geq \frac{\hat{\zeta}^2(N_{p}T_{p}-1)}{16}$. Then with probability at least $1-5\delta$, the weighted least squares estimate $\Theta_{WLS}$ using $\lambda=0$ satisfies 
\begin{equation} \label{eqn:upper bound on error}
\begin{aligned}
&\|\Theta_{WLS}-\Theta\|\\
& \leq \underbrace{\frac{20\max(\sigma_{\bar{w}}, \sqrt{q}\sigma_{\hat{w}})\sqrt{\log\frac{9^{n}}{\delta}+(n+p)\log(\phi)}}{\sqrt{N_{r}T_{r}\bar{\zeta}^2+qN_{p}T_{p}\hat{\zeta}^2}}}_\text{Error due to noise}\\
&+\underbrace{q\|\delta_{\Theta}\|\frac{33\hat{g}(\delta)}{N_{r}T_{r}\bar{\zeta}^2+qN_{p}T_{p}\hat{\zeta}^2}}_\text{Error due to difference between true and auxiliary systems},
\end{aligned}
\end{equation}
where 
\begin{equation*}
\begin{aligned}
\phi=\phi(N_{r}, T_{r},N_{p},T_{p},q)=\frac{33(\bar{g}(\delta)+q\hat{g}(\delta))}{N_{r}T_{r}\bar{\zeta}^2+qN_{p}T_{p}\hat{\zeta}^2}+1,
\end{aligned}
\end{equation*}
\begin{equation*}
\begin{aligned}
\bar{g}(\delta)=N_{r}\sum_{t=0}^{T_{r}-1}(\tr(\bar{G}_{t})+p)(\frac{1}{c}\log(\frac{2}{\delta})+1)\bar{\sigma}_{max}^2,
\end{aligned}
\end{equation*}
\begin{equation*}
\begin{aligned}
\hat{g}(\delta)=N_{p}\sum_{t=0}^{T_{p}-1}(\tr(\hat{G}_{t})+p)(\frac{1}{c}\log(\frac{2}{\delta})+1)\hat{\sigma}_{max}^2.
\end{aligned}
\end{equation*}
\end{theorem}

\begin{remark} \textbf{Interpretation of Theorem \ref{data-independent bound}}. Recall that $N_{r}$ is the number of rollouts from the true system \eqref{eq:True system}, $T_{r}$ is the length of each rollout of the true system \eqref{eq:True system}, $N_{p}$ is the number of rollouts from the auxiliary system \eqref{eq:Perturbed system}, and $T_{p}$ is the length of each rollout of the auxiliary system \eqref{eq:Perturbed system}. Consequently, the quantities $N_{r}T_{r}$ and $N_{p}T_{p}$ capture the total number of samples from the true system and the auxiliary system, respectively. Further, recall that $\sigma_{\bar{w}}, \sigma_{\hat{w}}$ capture the noise levels from the two systems, and $\|\delta_{\Theta}\|$ captures the difference between the two system models. For strictly stable systems \eqref{eq:True system}-\eqref{eq:Perturbed system}, $\bar{g}(\delta)$ and $\hat{g}(\delta)$ grow at most linearly with respect to $T_{r}$ and $T_{p}$. For marginally stable systems, $\bar{g}(\delta)$ and $\hat{g}(\delta)$ grow at most polynomially with respect to $T_{r}, T_{p}$ (see Proposition \ref{prop:system matrices bound} in the appendix). The terms $\bar{g}(\delta)$ and $\hat{g}(\delta)$ grow at most linearly with respect to $N_r,N_{p}$, irrespective of the spectral radius of the systems. Consequently, the parameter $\phi$ remains bounded with respect to $T_{r},T_{p}$ for strictly stable systems, grows at most polynomially with respect to $T_{r}, T_{p}$ for marginally stable systems, and remains bounded with respect to $N_{r},N_{p}$, irrespective of the spectral radius of the systems.  We discuss some further observations below.


\textit{Error due to noise and error due to model difference:} Theorem \ref{data-independent bound} decomposes the overall estimation error into the error due to noise (or the intrinsic error) and the error due to model difference. Suppose that $q=1$, and both systems are strictly stable for now. The error due to noise depends on the noise levels from the true system and the auxiliary system, and can be reduced by increasing the number of samples from the true system and the auxiliary system (increase $N_{r}T_{r}$ or $N_{p}T_{p}$). Theorem~\ref{data-independent bound} is an improvement over the result in \cite{xin2022identifying}, since Theorem~\ref{data-independent bound} shows that one can reduce the error due to noise by increasing either the number of rollouts or the length of these rollouts (or both), whereas the result in \cite{xin2022identifying} only shows the error due to noise can be reduced by increasing the number of rollouts. The dependence on $\sqrt{n+p}$ is due to the dimension of the system model we wish to learn. The error due to model difference depends on how similar the two systems are, and becomes smaller if the auxiliary system is more similar to the true system (smaller $\|\delta_{\Theta}\|$), or if there are more samples from the true system than auxiliary system (increase $N_{r}T_{r}$). Consequently, one can observe that increasing the number of samples from the auxiliary system helps to reduce the error due to noise, at the price of adding a portion of error due to model difference (note that the error due to model difference is always bounded when we increase $N_{p}$ or $T_{p}$). In particular, when the two systems are exactly the same, i.e., $ \|\delta_{\Theta}\|=0$, Theorem \ref{data-independent bound} recovers the learning rate $\mathcal{O}(\frac{1}{\sqrt{N_{r}T_{r}+N_{p}T_{p}}})$, which qualitatively matches the learning rate as reported in \cite{tu2022learning}, when all samples are generated from the same system.

When the two systems are both marginally stable, one can see that the error due to noise can still go to zero as $T_{r},T_{p}$ increase, since the term $\phi$ grows at most polynomially with respect to $T_{r}, T_{p}$. However, the error due to model difference may amplify as $T_{p}$ increases. We provide a slightly refined bound for large $N_{p}T_{p}$ in Proposition \ref{data-independent bound refined} to capture this case.

\textit{The benefits of collecting multiple trajectories:} The existing literature has shown that the multiple trajectories setup has the benefit of handling unstable systems (when all samples are collected from the true system), since restarting the system from an initial state prevents the system state from going to infinity over time \cite{dean2019sample}. This benefit is captured by our result. In particular, fixing $T_{r},T_{p}, q$, one can observe that the error due to noise always goes to zero as we increase $N_{r}$ or $N_{p}$, irrespective of the spectral radius of the two systems, since the parameter $\phi$ is bounded. Further, the error due to model difference always goes to zero as we increase $N_{r}$, 
and remains constant as we increase $N_{p}$, again irrespective of the spectral radius of the two systems.

\textit{The selection of weight parameter q:} In practice, selecting a good weight parameter $q$ based on Theorem \ref{data-independent bound} requires an oracle, since one has to know the specific values of the different parameters in Theorem \ref{data-independent bound}. Further, due to the different realizations of random variables, the optimal weight might differ at each experiment. A commonly used approach for tuning parameters in the training process is to leverage a cross-validation process (see \cite{refaeilzadeh2009cross} for an overview). In section \ref{sec:data-dependent}, we also provide a data-dependent bound, which is computable and can help one to select a good value of $q$ based on data. However, general guidelines can be given based on the upper bound provided by Theorem \ref{data-independent bound} when $N_{p}$ or $T_{p}$ is large and $\|\delta_{\Theta}\|$ is small, supposing that the two systems are strictly stable (for simplicity):
\begin{itemize}
    \item When $N_{r}T_{r}$ is small, we can set $q$ to be relatively large to make sure that the first term in the error bound is small (use the auxiliary data to reduce the error due to noise). Consequently, the error bound is essentially dominated by the second term, which is small if the two systems are similar. This corresponds to the case where we have little data from the true system, and thus there may be a large identification error due to using only that data. 
    \item When $N_{r}T_{r}$ is large, we can  decrease $q$ to reduce the second term as well, since the first term is already made small enough. This corresponds to the case where we have a large amount of data from the true system, and only need the data from the auxiliary system to slightly improve our estimates. In this case, we place a lower weight on the auxiliary data in order to avoid excessive bias due to the difference in the dynamics of the two systems.
    \end{itemize}
Furthermore, Theorem \ref{data-independent bound} demonstrates how the weight parameter should scale. For example, it can be verified that one can set $q=\mathcal{O}(\frac{1}{\sqrt{N_{r}}})$ to ensure consistency, when $N_{p}$ grows linearly with respect to $N_{r}$ (when $T_{r}, T_{p}$ are fixed). These ideas will also be illustrated experimentally in Section \ref{exp}.
\end{remark}

Finally, the following corollary of Theorem \ref{data-independent bound} provides a sufficient condition under which using the data from the auxiliary system (setting $q\neq 0$) leads to a smaller error bound compared to using data only from the true system (setting $q=0$), when both the true system and the auxiliary system are strictly stable.

\begin{coro}\label{When to use}
Suppose that both system \eqref{eq:True system} and system \eqref{eq:Perturbed system} are strictly stable, i.e., $\rho(\bar{A})<1$ and $\rho(\hat{A})<1$.  Consider the estimation error bound provided in Theorem \ref{data-independent bound}. Suppose that $q$ satisfies the following inequality: 
\begin{equation}
\begin{aligned}
& \frac{\sigma_{\bar{w}}\sqrt{\log\frac{9^{n}}{\delta}+(n+p)\log(\frac{33\bar{g}(\delta)}{N_{r}T_{r}\bar{\zeta}^2}+1)}}{\sqrt{N_{r}T_{r}\bar{\zeta}^2}}\\
& > \frac{\max(\frac{\sigma_{\bar{w}}}{\sqrt{q}}, \sigma_{\hat{w}})\sqrt{\log\frac{9^{n}}{\delta}+(n+p)\log(\frac{\gamma}{\zeta^2}+1)}}{\sqrt{N_{p}T_{p}\hat{\zeta}^2}}\\
&+\|\delta_{\Theta}\|\frac{\gamma}{20\hat{\zeta}^2},
\label{eq:condition_to_use_aux_data}
\end{aligned}
\end{equation}
where $\zeta=\min(\bar{\zeta}, \hat{\zeta})$, and $\gamma$ is any positive constant that satisfies
\begin{equation*} 
\begin{aligned}
&\max(33(\tr(\bar{G}_{t})+p)(\frac{1}{c}\log(\frac{2}{\delta})+1)\bar{\sigma}_{max}^2,\\
&\quad\quad\quad\quad\quad\quad33(\tr(\hat{G}_{t})+p)(\frac{1}{c}\log(\frac{2}{\delta})+1)\hat{\sigma}_{max}^2)\leq \gamma
\end{aligned}
\end{equation*}
for all $t\geq 0$. Then the resulting error bound will be smaller than the error bound obtained using $q=0$. 
\end{coro}

\begin{remark} \textbf{Interpretation of Corollary \ref{When to use}}.
Note that $\gamma$ always exists since $\tr(\bar{G}_{t})$ and $\tr(\hat{G}_{t})$ are bounded for strictly stable systems (see Proposition \ref{prop:system matrices bound} in the appendix). We also note that the above condition might be conservative, and may not be easily checked in practice since it involves unknown parameters. However, we describe the insights provided by this condition here. One may observe that condition \eqref{eq:condition_to_use_aux_data} is more likely to hold if $\|\delta_{\Theta}\|$ is small (the true system and the auxiliary system shares ``similar" dynamics), and $N_{p}T_{p}$ is large (one has abundant samples from the auxiliary system), as these conditions can make the right hand side of the inequality smaller. In other words, the auxiliary samples tend to be more informative in such cases. In contrast, condition \eqref{eq:condition_to_use_aux_data} is less likely to hold if $N_{r}T_{r}$ is large, since it will make the left hand side of the inequality smaller, i.e., if we already have a lot of samples from the true system, then the auxiliary samples tend to be less informative. The effect of the noise can be quite subtle, since it shows up in various places. However, loosely speaking, having a smaller $\sigma_{\hat{w}}$ while assigning higher weight $q$ may still help to make the right hand side of the inequality smaller by making the term $\max(\frac{\sigma_{\bar{w}}}{\sqrt{q}},\sigma_{\hat{w}})$ smaller, when the terms $\hat{\zeta}$ and $\bar{\zeta}$ are not affected too much. In other words, we might be able to benefit from the auxiliary system if the auxiliary system is not too noisy, and if we attach enough importance to the auxiliary samples.  
\end{remark}

The following result is a slightly refined bound for large $N_{p}T_{p}$.
\begin{proposition}
\label{data-independent bound refined}
Under the same conditions in Theorem \ref{data-independent bound}, with probability at least $1-5\delta$, the weighted least squares estimate $\Theta_{WLS}$ using $\lambda=0$ satisfies 
\begin{equation} 
\begin{aligned}
&\|\Theta_{WLS}-\Theta\|\\
& \leq \underbrace{\frac{20\max(\sigma_{\bar{w}}, \sqrt{q}\sigma_{\hat{w}})\sqrt{\log\frac{9^{n}}{\delta}+(n+p)\log(\phi)}}{\sqrt{N_{r}T_{r}\bar{\zeta}^2+qN_{p}T_{p}\hat{\zeta}^2}}}_\text{Error due to noise}\\
&+\underbrace{\|\delta_{\Theta}\|(1+\frac{33\bar{g}(\delta)}{N_{r}T_{r}\bar{\zeta}^2+qN_{p}T_{p}\hat{\zeta}^2})}_\text{Error due to difference between true and auxiliary systems}.
\end{aligned}
\end{equation}
\end{proposition}
\begin{remark}
Proposition \ref{data-independent bound refined} yields a refined bound when $N_{p}$ or $T_{p}$ goes to infinity (compared to Theorem \ref{data-independent bound}). Fixing $q>0$, it can be observed that the error bound converges exactly to $\|\delta_{\Theta}\|$ as $T_{p}$ increases, when the auxiliary system is marginally stable or strictly stable ($\rho(\hat{A})\leq 1$), since the $\phi$ term grows at most polynomially with respect to $T_{p}$, and the error due to model difference converges to $\|\delta_{\Theta}\|$. Further, the bound converges exactly to $\|\delta_{\Theta}\|$ as $N_{p}$ increases, irrespective of the spectral radius of the system. In other words, one is essentially learning the dynamics of the auxiliary system when we use a lot of auxiliary samples. 
\end{remark}
\subsection{Data-dependent Bound} \label{sec:data-dependent}
In this section, we provide a data-dependent upper bound of the system identification error, assuming $\lambda>0$. The regularized solution with strictly positive $\lambda$ helps us to establish the data-dependent bound, and provides the user with more flexibility to tune the estimate in practice. The bound is computable when some prior knowledge about the systems (as will be discussed below) is available, and applies to more general input and noise. One can also use it for the selection of weight parameter $q$ and regularization parameter $\lambda$ in practice (by selecting the weight parameter $q$ and regularization parameter $\lambda$ that give a smaller error bound).  
\begin{theorem}
\label{data-dependent bound}
Consider the systems \eqref{eq:True system}-\eqref{eq:Perturbed system},
where the random vectors $\bar{w}_{t},\bar{u}_{t},\bar{x}_{0}, \hat{w}_{t}$, $\hat{u}_{t},\hat{x}_{0}$ are independent, and $\bar{w}_{t},\hat{w}_{t}$ are sub-Gaussian with parameters $\sigma_{\bar{w}}^2$ and $\sigma_{\hat{w}}^2$, respectively, for all $t\geq 0$. Fix $q\geq 0$, $\lambda>0$, and $\delta >0$. Let $V=\lambda I_{n+p}$ and $\bar{V}=(ZQZ'+V)V^{-1}$. With probability at least $1-\delta$, the regularized weighted least squares estimate $\Theta_{WLS}$ satisfies 
\begin{equation}  \label{eq:data-dependent}
\begin{aligned}
&\|\Theta_{WLS}-\Theta\|\leq\\
& \frac{\max(\sigma_{\bar{w}}, \sqrt{q}\sigma_{\hat{w}})\sqrt{\frac{32}{9}(\log\frac{9^{n}}{\delta}+\frac{1}{2}\log\det(\bar{V}))}}{\sqrt{\lambda_{min}(ZQZ'+\lambda I_{n+p})}}\\
&+q\|\delta_{\Theta}\|\|\hat{Z}\hat{Z}'(ZQZ'+\lambda I_{n+p})^{-1}\|\\
&+\|\Theta\|\frac{\lambda}{\lambda_{min}(ZQZ'+\lambda I_{n+p})}.
\end{aligned}
\end{equation}
\end{theorem}


\begin{remark} \label{model difference} Practically, one can compute the error bound in Theorem~\ref{data-dependent bound} using various different $q$ and $\lambda$, and choose a value of $q$ and $\lambda$ that give the smallest error bound. We will illustrate the selection of $q$ in Section \ref{exp}. Note that the model difference term $\|\delta_{\Theta}\|$ in \eqref{eq:data-dependent} can be replaced by an upper bound on the difference between the models (if that is available). In practice, an appropriate upper bound of the term $\|\delta_{\Theta}\|$ could be obtained using prior knowledge or previous estimates from data. For example, if one knows that an auxiliary system is different from the true system only in certain subsystems, where the entries are restricted to a fixed range, that information can be used to compute an upper bound of $\|\delta_{\Theta}\|$. Also, the difference between the subsystems of the true system and the auxiliary can be estimated from data using any existing techniques (e.g., least squares method). This may be helpful if doing experiments on the subsystems is easy. The bound on $\|\Theta\|$ can be obtained similarly using prior knowledge (e.g., using known range on entries). The noise distribution of the two systems and their corresponding sub-Gaussian parameters can also be estimated from data \cite{silverman2018density}.
\end{remark}



\section{Numerical Experiments to Illustrate Various Scenarios for System Identification from Auxiliary Data}
\label{exp}
In order to validate our main results in Theorem \ref{data-independent bound} and Theorem \ref{data-dependent bound} and  gain more insights, we now provide numerical examples of the weighted least squares-based system identification algorithm. All of the numerical results are averaged over $10$ independent experiments. 

\subsection{Predetermined $q$} \label{expA}
In this section, we provide numerical experiments using various predetermined weight parameters $q$. Such a situation may occur if we have a firm belief that the auxiliary system has similar dynamics to the true system, but upper bounds on $\|\delta_{\Theta}\|,\sigma_{\bar{w}}$ and $\sigma_{\hat{w}}$ are not available. Setting $\lambda=0$, the experiments are performed using the following true system and auxiliary system:
\begin{equation}
\begin{aligned}
\bar{A}=
\begin{bmatrix}
0.6&0.5&0.4\\
0&0.5&0.4\\
0&0&0.4\\
\end{bmatrix},
\indent 
\bar{B}=
\begin{bmatrix}
1&0.5\\
0.5&1\\
0.5&0.5\\
\end{bmatrix},
\end{aligned}
\end{equation}
\begin{equation}
\begin{aligned}
\hat{A}=
\begin{bmatrix}
0.7&0.5&0.4\\
0&0.5&0.4\\
0&0&0.4\\
\end{bmatrix},
\indent 
\hat{B}=
\begin{bmatrix}
1.1&0.5\\
0.5&1\\
0.5&0.5\\
\end{bmatrix}.
\end{aligned}
\end{equation}
We set $\bar{x}_{0}, \hat{x}_{0}, \bar{u}_{t}, \hat{u}_{t}, \bar{w}_{t}, \hat{w}_{t}$ to be zero mean Gaussian random vectors with covariance matrices being identity matrices. The model difference of the above two systems is $\|\delta_{\Theta}\| \approx 0.1414$. The numbers of rollouts $N_{r}$ and $N_{p}$ are set to be $1$. We provide experiments to illustrate various scenarios, including those we mentioned earlier in Section \ref{sec: introduction}. Note that the experiments in this section are conducted by varying the lengths of the trajectories from the two systems, while the experiments in \cite{xin2022identifying} are performed by varying the number of trajectories from the two systems. 

\begin{figure*}
\minipage[t]{0.32\textwidth} 
    \includegraphics[width=\linewidth]{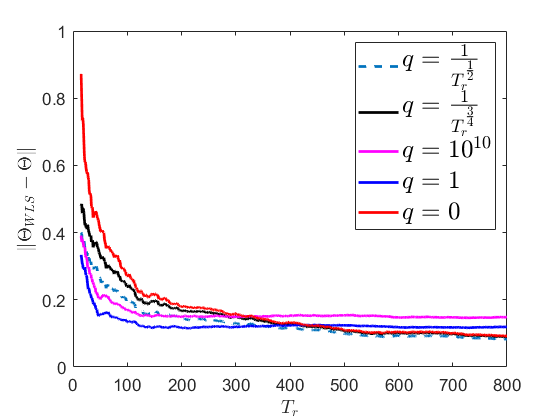}
    \caption{Scenario 1: Both $T_r$ and $T_p$ increase over time ($T_p = 3T_r$)}
    \label{fig:sce1} 
\endminipage \hfill
\minipage[t]{0.32\textwidth}
    \includegraphics[width=\linewidth]{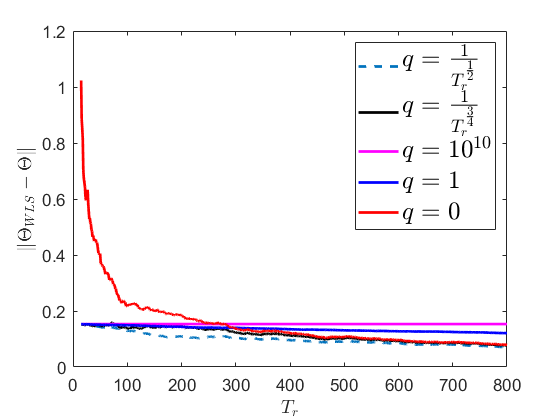}
    \caption{Scenario 2: $T_p$ is fixed, and $T_r$ increases over time}
    \label{fig:sce2} 
\endminipage \hfill
\minipage[t]{0.32\textwidth}
    \includegraphics[width=\linewidth]{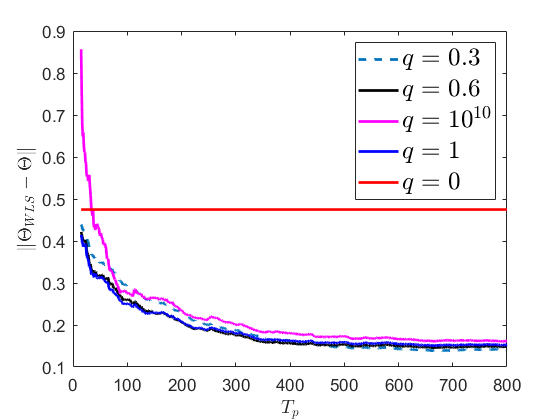}
    \caption{Scenario 3: $T_r$ is fixed, and $T_p$ increases over time}
    \label{fig:sce3} 
\endminipage
\end{figure*}

\subsubsection{Scenario 1: Both $T_{r}$ and $T_{p}$ are increasing}

In the first experiment, we set the length of the trajectory from the auxiliary system be $T_{p}=3T_{r}$. In practice, one may encounter such a scenario when gathering data from the true system is time consuming or costly, whereas gathering data from an auxiliary system (such as a simulator) is faster or cheaper.  

In Fig.~\ref{fig:sce1}, we plot the estimation error $\|\Theta-
\Theta_{WLS}\|$ versus $T_{r}$ using different weight parameters $q$. As expected, when one does not have enough data from the true system ($T_{r}$ is small), setting $q>0$ leads to a smaller estimation error of system matrices. However, the curve for $q=1$ and $q=10^{10}$ (corresponding to treating all samples equally and paying almost no attention to the samples from the true system, respectively) eventually plateau and incur more error than not using the the auxiliary data ($q=0$). This phenomenon matches with the theoretical guarantee in Theorem \ref{data-independent bound}. Specifically, when $q$ is a nonzero constant and  both $T_p$ and $T_r$ are increasing in a linear relationship, it can be verified that the upper bound in Theorem \ref{data-independent bound} will not go to zero as $T_{r}$ increases. 
In contrast, setting $q$ to be diminishing with $T_{r}$ could perform consistently better than $q=0$ in this example, even when $T_{r}$ becomes large. Indeed, one can choose $q=\mathcal{O}(\frac{1}{\sqrt{T_{r}}})$ in the upper bound given by \eqref{eqn:upper bound on error} in Theorem \ref{data-independent bound}, and show that the upper bound becomes $\mathcal{O}(\frac{1}{\sqrt{T_{r}}})$. 

\subsubsection{Scenario 2: $T_{p}$ is fixed but $T_{r}$ is increasing}
For the second experiment, we fix the number of samples from the auxiliary system to be $T_{p}=2400$, and look at what happens as the number of samples from the true system increases. In practice, one may encounter such a scenario when the system dynamics change at some point in time (e.g., due to faults).  In this case, the true system we want to learn is the one after the fault, and the auxiliary system is the one prior to the fault. Consequently, while the data from the old (auxiliary) system may not accurately represent the new (true) system dynamics, leveraging the old data might be beneficial in this case.

In Fig.~\ref{fig:sce2}, we plot the estimation error versus $T_{r}$ for different weight parameters $q$. As expected, setting $q>0$ leads to a much smaller error during the initial phase when $T_{r}$ is small. This can be confirmed by Theorem \ref{data-independent bound} since the overall estimation error is essentially the error due to the model difference.  Namely, the auxiliary data helps to build a good initial estimate when $T_{r}$ is small. When we set the weight to be $q=10^{10}$, we are paying little attention to the samples from the true system, i.e., we are not gaining any new information as we collect more data from the true system. Consequently, the error is almost a flat line as $T_{r}$ increases when $q=10^{10}$. As can be observed from Theorem \ref{data-independent bound}, when $T_{p}$ is fixed, we can always make the error go to $0$ as we increase $T_{r}$, using the weights we selected in this experiment. However, when $q$ is set to be too large, it could make the error even larger due to the model difference (or bias) introduced by the auxiliary system. This is captured by Theorem \ref{data-independent bound} since when $q$ is set to be too large (such that $qT_{p}$ is large compared to $T_{r}$), even when $T_{r}$ becomes larger, the second term in the error bound \eqref{eqn:upper bound on error} (capturing model difference) is still large.

\subsubsection{Scenario 3: $T_{r}$ is fixed but $T_{p}$ is increasing}
In the last experiment, we fix the number of samples from the true system to be $T_{r}=50$.  As discussed earlier, one may encounter such a scenario when one has only a limited amount of time to gather data from the true system. Consequently, leveraging information from other ``similar'' systems (e.g., from a reasonably accurate simulator) could be helpful to augment the data.  This is the most subtle case, since Theorem \ref{data-independent bound} does not ensure consistency when $T_{r}$ is fixed.

In Fig.~\ref{fig:sce3}, we plot the the estimation error versus $T_{p}$ using different weight parameters $q$.  As it can be seen, setting $q=0$ (not using the auxiliary samples) gives a flat line, which represents the error we can achieve purely based on $T_{r}=50$ samples from the true system. When $q=10^{10}$, we are paying little attention to the true system, and essentially learning the dynamics of the auxiliary system. In contrast, the results for $q=1, 0.6, 0.3$ suggest that setting a relatively balanced weight $q$ to the auxiliary data could make the error smaller than the two extreme cases ($q=0, 10^{10}$) in this example. However, in practice, one may want to leverage a cross-validation process to tune the hyper-parameter $q$, when there is not enough prior knowledge about the dynamics of the true system and the auxiliary system.

\subsection{Selecting $q$ based on Theorem \ref{data-dependent bound}} \label{expB}
In this section, we study selecting the weight parameter $q$ using Theorem \ref{data-dependent bound} using a fixed regularization parameter $\lambda$. We plot the true error $\|\Theta-\Theta_{WLS}\|$ and the theoretical data-dependent bound in Theorem \ref{data-dependent bound} as a function of weight parameter $q$ varying from $0$ to $2$, where the increment is set to be $0.01$. This corresponds to the situation where some upper bounds on $\|\delta_{\Theta}\|,\sigma_{\bar{w}}$ and $\sigma_{\hat{w}}$ are available (as discussed in Remark \ref{model difference}). We set the confidence parameter to be $\delta=0.01$, and all other parameters in Theorem \ref{data-dependent bound} are assumed to be known exactly for simplicity. The system matrices of the true system and the auxiliary system are set to be

\begin{equation}
\begin{aligned}
\bar{A}=
\begin{bmatrix}
0.6&0.5&0.4\\
0&0.5&0.4\\
0&0&0.4\\
\end{bmatrix},
\indent 
\bar{B}=
\begin{bmatrix}
1&0.5\\
0.5&1\\
0.5&0.5\\
\end{bmatrix},
\end{aligned}
\end{equation}
\begin{equation}
\begin{aligned}
\hat{A}=
\begin{bmatrix}
0.7&0.5&0.4\\
0&0.5&0.4\\
0&0&0.4\\
\end{bmatrix},
\indent 
\hat{B}=
\begin{bmatrix}
1+\mathbf{\Delta}&0.5\\
0.5&1\\
0.5&0.5\\
\end{bmatrix}.
\end{aligned}
\end{equation}
We set the regularization parameter to be $\lambda=1$. We set $\bar{x}_{0}, \hat{x}_{0}, \bar{u}_{t}, \hat{u}_{t}, \bar{w}_{t}, \hat{w}_{t}$ to be zero mean Gaussian random vectors, where the covariance matrices of $\bar{x}_{0}, \hat{x}_{0}, \bar{u}_{t}, \hat{u}_{t}$ are set to be identity matrices. The trajectory lengths are set to be $T_{r}=10, T_{p}=50$, and the number of trajectories from the auxiliary system is set to be $N_{p}=20$. The covariance matrices of $\bar{w}_{t}, \hat{w}_{t}$ are set to be $\sigma^2_{\bar{w}}I_{n+p}, \sigma^2_{\hat{w}}I_{n+p}$, and the values of $\mathbf{\Delta}, \sigma_{\bar{w}}, \sigma_{\hat{w}}, N_{r}$ are specified under the figures. 
\begin{figure*}
\minipage[t]{0.32\textwidth} 
    \includegraphics[width=\linewidth]{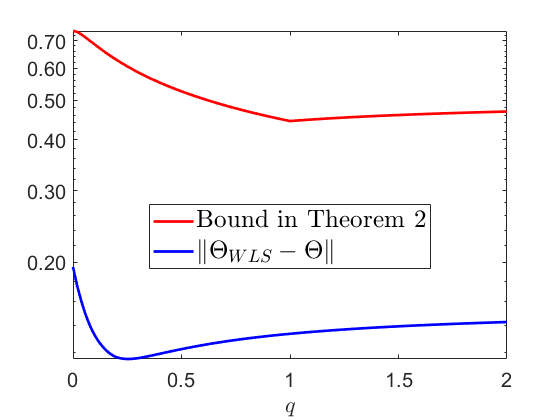}
    \caption{Baseline case 1: $\mathbf{\Delta}=0.1, \sigma_{\bar{w}}=\sigma_{\hat{w}}=1, N_{r}=20$. An intermediate value of $q$ is optimal}
    \label{fig:moderate1} 
\endminipage \hfill
\minipage[t]{0.32\textwidth} 
    \includegraphics[width=\linewidth]{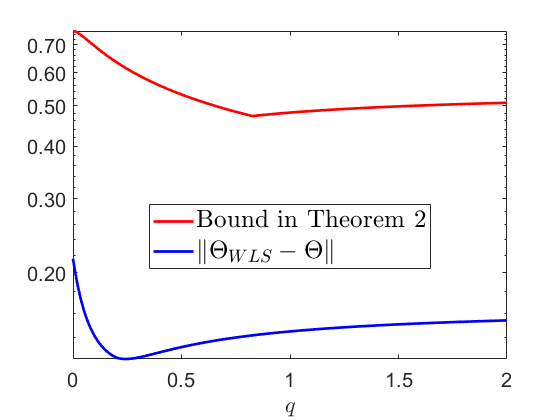}
    \caption{Baseline case 2: $\mathbf{\Delta}=0.11, \sigma_{\bar{w}}=1, \sigma_{\hat{w}}=1.1, N_{r}=19$. An intermediate value of $q$ is optimal}
    \label{fig:moderate2} 
\endminipage  \hfill
\minipage[t]{0.32\textwidth}
    \includegraphics[width=\linewidth]{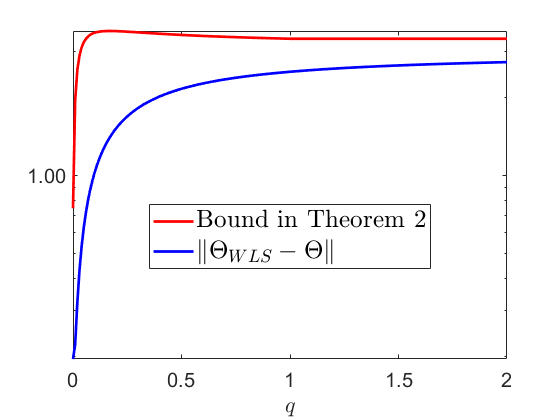}
    \caption{Large model difference: $\mathbf{\Delta}=3, \sigma_{\bar{w}}=\sigma_{\hat{w}}=1, N_{r}=20$. In this case, it is optimal to not use data from the auxiliary system ($q=0$)}
    \label{fig:largebias} 
\endminipage \hfill

\minipage[t]{0.32\textwidth}
    \includegraphics[width=\linewidth]{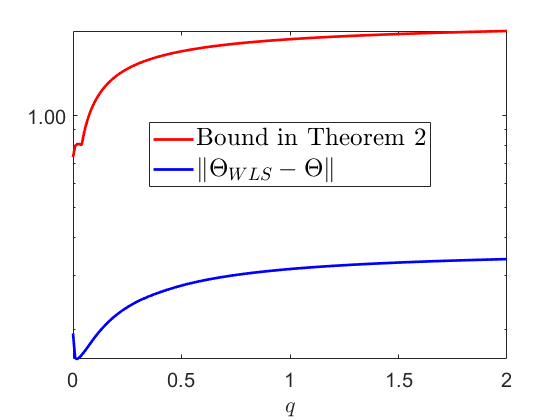}
    \caption{Noisy auxiliary system: $\mathbf{\Delta}=0.1, \sigma_{\bar{w}}=1, \sigma_{\hat{w}}=5, N_{r}=20$. In this case, it is optimal to not use data from the auxiliary system ($q=0$)}
    \label{fig:largenoise} 
\endminipage \hfill
\minipage[t]{0.32\textwidth}
    \includegraphics[width=\linewidth]{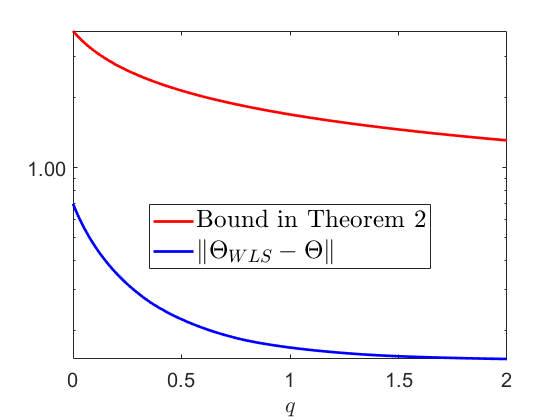}
    \caption{Noisy true system: $\mathbf{\Delta}=0.1, \sigma_{\bar{w}}=5, \sigma_{\hat{w}}=1, N_{r}=20$. In this case, it is optimal to assign higher weight to the auxiliary system ($q=2$)}
    \label{fig:noisyTrue} 
\endminipage \hfill
\minipage[t]{0.32\textwidth} 
    \includegraphics[width=\linewidth]{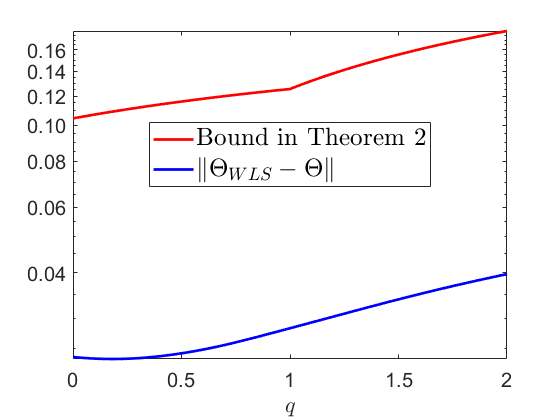}
    \caption{Large number of true samples: $\mathbf{\Delta}=0.1, \sigma_{\bar{w}}=\sigma_{\hat{w}}=1, N_{r}=1200$. In this case, it is optimal to not use data from the auxiliary system ($q=0$)}
    \label{fig:largeNr} 
\endminipage  \hfill
\end{figure*}

As can be seen in Fig.~\ref{fig:moderate1} and Fig.~\ref{fig:moderate2}, setting $q$ to be non-zero could result in smaller error bounds, which show the benefits of leveraging the auxiliary data. However, for the values of the weight $q$ we plotted, the optimal weights that obtain the smallest error bounds in Theorem \ref{data-dependent bound} do not align with the optimal weights $q$ that minimize the true error $\|\Theta-\Theta_{WLS}\|$. Such mismatches could be due to the conservativeness of the bound. On the other hand, the optimal weight from the bound still captures how the true optimal weight should scale. As can be seen in Fig.~\ref{fig:largebias}, Fig.~\ref{fig:largenoise}, and Fig.~\ref{fig:largeNr}, the optimal weight for both the bound and the true error tend to be small when (1) the model difference is large; or (2) the auxiliary system becomes much more noisy; or (3) when one has a large number of samples from the true system. Such empirical results also match with our observations in Corollary \ref{When to use} when $\lambda=0$. In Fig.~\ref{fig:noisyTrue}, both the optimal weight from our bound and the true optimal weight are greater than $1$, since the true system is very noisy and hence the data from the true system tend to be less informative compared to the data from the auxiliary system. We further note that, in practice, selecting the exact optimal weight $q$ that minimizes $\|\Theta_{WLS}-\Theta\|$ is very hard, and one would instead focus on selecting a relatively good weight $q$. The weight that results in a small error bound can be integrated with techniques like robust control to improve the overall system performance guarantee.   

\section{Conclusion}
\label{conclusion}
In this paper, we provided finite sample analysis of system identification using a weighted least squares approach, when one has an auxiliary system that shares similar dynamics as the true system we want to learn. The analysis improves the result in \cite{xin2022identifying} as we show the error due to noise can be reduced by increasing either the number of trajectories or the trajectory length of the true system and the auxiliary system, or both. Our analysis provides insights on the benefits of using the auxiliary system, and how to weight the data from the auxiliary system. We also provided a data-dependent bound that is computable when some prior knowledge about the systems is available, which is tighter and can be used to determine the appropriate weight parameter in the training process. 

There are various directions for future research. First, as shown in \cite{sarkar2019near, faradonbeh2018finite}, the least squares estimator is consistent for certain types of unstable systems even if multiple trajectories of data are not available. It would be interesting to study how to capture that in our analysis. Second, it would be of interest to relax the conditions on the full rankness of the covariance matrix of noise/input in our setup such that one could handle systems with longer memory. Another interesting direction is to develop lower bounds for such transfer learning-based system identification methods, which could potentially enable the development of an optimal estimator. Some possible approaches are to leverage assumptions like sparsity \cite{bastani2021predicting} or prior knowledge, e.g., known auxiliary system model. Finally, studying how to leverage the idea of learning from similar systems/transfer learning in control-related problems would also be a rich area for future research\cite{li2022data}.

\section*{Appendix} 
\subsection{Intermediate Results}
We will leverage the following Hanson-Wright inequality to upper bound the terms $\|\bar{Z}\bar{Z}'\|$ and $\|\hat{Z}\hat{Z}'\|$.
\begin{lemma}\cite[Theorem~1.1]{rudelson2013hanson}
\label{hanson wright}
Let $X=\begin{bmatrix} X_{1}& \ldots X_{n} \end{bmatrix}'\in \mathbb{R}^{n}$
be a random vector with independent components $X_{i}$ which satisfy $\mathbb{E}[X_{i}]=0$ and $\|X_{i}\|_{\psi_{2}}\leq K$, where $\|\cdot\|_{\psi_{2}}$ denotes the sub-Gaussian norm, i.e., $\|X_{i}\|_{\psi_{2}}=\inf\{\zeta>0: \mathbb{E}[\exp(X_{i}^2/\zeta^2)]\leq 2\}$. Let $A$ be an $n \times n$ matrix. Then, for every $t\geq0$, we have
\begin{equation*}
\begin{aligned}
&P(|X'AX-\mathbb{E}[X'AX]|>t)\\ &\leq 2\exp(-c_{0}\min(\frac{t^2}{K^4\|A\|^{2}_{F}},\frac{t}{K^2\|A\|})),\\
\end{aligned}
\end{equation*}
where $c_{0}$ is some positive universal constant.
\end{lemma}

We have the following result.
\begin{lemma}
Let Assumption \ref{assumption} hold. For any fixed $\delta\in(0,\frac{2}{e})$, each of the following inequalities holds with probability at least $1-\delta$:
\begin{equation*}
\|\sum_{i=1}^{N_{r}}\sum_{t=0}^{T_{r}-1}\bar{z}^{i}_{t}\bar{z}^{i'}_{t}\|\leq \bar{g}(\delta),
\end{equation*}
\begin{equation*}
\|\sum_{i=1}^{N_{r}}\sum_{t=0}^{T_{p}-1}\hat{z}^{i}_{t}\hat{z}^{i'}_{t}\|\leq \hat{g}(\delta),
\end{equation*}
where
\begin{equation*}
\bar{g}(\delta)=N_{r}\sum_{t=0}^{T_{r}-1}(\tr(\bar{G}_{t})+p)(\frac{1}{c}\log(\frac{2}{\delta})+1)\bar{\sigma}_{max}^2,
\end{equation*}
\begin{equation*}
\hat{g}(\delta)=N_{p}\sum_{t=0}^{T_{p}-1}(\tr(\hat{G}_{t})+p)(\frac{1}{c}\log(\frac{2}{\delta})+1)\hat{\sigma}_{max}^2,
\end{equation*}
\label{lemma:bound z}
and $c$ is some positive constant.
\end{lemma}
\begin{proof}
We will only show the first inequality since the analysis for the second one is essentially the same. 
Let $\mathbf{Z}_{T_{r}}^{N_{r}}=\begin{bmatrix}\mathbf{Z}^{1'}& \ldots& \mathbf{Z}^{N_{r}'} \end{bmatrix}'\in \mathbb{R}^{(n+p)N_{r}T_{r}}$, where $\mathbf{Z}^{i}=\begin{bmatrix} \bar{z}_{0}^{i'}& \ldots& \bar{z}_{T_{r}-1}^{i'} \end{bmatrix}'\in \mathbb{R}^{(n+p)T_{r}}$ for $i=1,\ldots,N_{r}$. We have 
\begin{equation}
\|\sum_{i=1}^{N_{r}}\sum_{t=0}^{T_{r}-1}\bar{z}_{t}^{i}\bar{z}^{i'}_{t}\|\leq \sum_{i=1}^{N_{r}}\sum_{t=0}^{T_{r}-1}\bar{z}_{t}^{i'}\bar{z}_{t}^{i}=\mathbf{Z}_{T_{r}}^{N_{r}'}\mathbf{Z}_{T_{r}}^{N_{r}}. \label{to_use5}
\end{equation} 
Note that we have
\begin{equation} \label{touse_5}
\begin{aligned}
\mathbf{Z}_{T_{r}}^{N_{r}'}\mathbf{Z}_{T_{r}}^{N_{r}}&\leq \mathbb{E}[\mathbf{Z}_{T_{r}}^{N_{r}'}\mathbf{Z}_{T_{r}}^{N_{r}}]+|\mathbf{Z}_{T_{r}}^{N_{r}'}\mathbf{Z}_{T_{r}}^{N_{r}}-\mathbb{E}[\mathbf{Z}_{T_{r}}^{N_{r}'}\mathbf{Z}_{T_{r}}^{N_{r}}]|.\\
\end{aligned}
\end{equation}
Now we will upper bound the two terms after the inequality in \eqref{touse_5}. We consider the term $\mathbb{E}[\mathbf{Z}_{T_{r}}^{N_{r}'}\mathbf{Z}_{T_{r}}^{N_{r}}]$ first. 
Let
\begin{equation*}
\begin{aligned}
H=
\begin{bmatrix}
H_{1}&H_{2}\end{bmatrix}\in\mathbb{R}^{^{(n+p)T_{r}\times (n+p)T_{r}}}, 
\end{aligned}
\end{equation*}
where $H_{1}\in\mathbb{R}^{(n+p)T_{r}\times nT_{r}}$ is defined as 
\begin{equation*}
\begin{aligned}
H_{1}=
\begin{bmatrix}
I_{n}&\mathbf{0}&\mathbf{0}&\cdots&\mathbf{0}&\mathbf{0}\\
\mathbf{0}&\mathbf{0}&\mathbf{0}&\cdots&\mathbf{0}&\mathbf{0}\\
\bar{A}&I_{n}&\mathbf{0}&\cdots&\mathbf{0}&\mathbf{0}\\
\mathbf{0}&\mathbf{0}&\mathbf{0}&\cdots&\mathbf{0}&\mathbf{0}\\
\bar{A}^2&\bar{A}&I_{n}&\cdots&\mathbf{0}&\mathbf{0}&\\
\mathbf{0}&\mathbf{0}&\mathbf{0}&\cdots&\mathbf{0}&\mathbf{0}\\
\vdots&\vdots&\vdots&\vdots&\vdots&\vdots&\\
\bar{A}^{T_{r}-1}&\bar{A}^{T_{r}-2}&\bar{A}^{T_{r}-3}&\cdots&\bar{A}&I_{n}\\
\mathbf{0}&\mathbf{0}&\mathbf{0}&\cdots&\mathbf{0}&\mathbf{0}\\
\end{bmatrix}, 
\end{aligned}
\end{equation*}
and $H_{2}\in\mathbb{R}^{(n+p)T_{r}\times pT_{r}}$ is defined as 
\begin{equation*}
\begin{aligned}
H_{2}=
\begin{bmatrix}
\mathbf{0}&\mathbf{0}&\mathbf{0}&\cdots&\mathbf{0}&\mathbf{0}\\
I_{p}&\mathbf{0}&\mathbf{0}&\cdots&\mathbf{0}&\mathbf{0}\\
\bar{B}&\mathbf{0}&\mathbf{0}&\cdots&\mathbf{0}&\mathbf{0}\\
\mathbf{0}&I_{p}&\mathbf{0}&\cdots&\mathbf{0}&\mathbf{0}\\
\bar{A}\bar{B}&\bar{B}&\mathbf{0}&\cdots&\mathbf{0}&\mathbf{0}&\\
\mathbf{0}&\mathbf{0}&I_{p}&\cdots&\mathbf{0}&\mathbf{0}\\
\vdots&\vdots&\vdots&\vdots&\vdots&\vdots&\\
\bar{A}^{T_{r}-2}\bar{B}&\bar{A}^{T_{r}-3}\bar{B}&\bar{A}^{T_{r}-4}\bar{B}&\cdots&\bar{B}&\mathbf{0}\\
\mathbf{0}&\mathbf{0}&\mathbf{0}&\cdots&\mathbf{0}&I_{p}\\
\end{bmatrix},
\end{aligned}
\end{equation*}
where we use $\mathbf{0}$ to denote zero matrices with appropriate dimensions.
Further, let $\mathbf{H}=\diag(H,\cdots,H)\in \mathbb{R}^{(n+p)N_{r}T_{r}\times (n+p)N_{r}T_{r}}$ and $g=\begin{bmatrix}
g^{1'}&g^{2'}&\cdots&g^{N_{r}'}\\
\end{bmatrix}'\in\mathbb{R}^{(n+p)N_{r}T_{r}}$, where $g^{i}=\begin{bmatrix}
\bar{x}_{0}^{i'}&\bar{w}_{0}^{i'}&\cdots&\bar{w}_{T_{r}-2}^{i'}&\bar{u}_{0}^{i'}&\cdots&\bar{u}_{T_{r}-1}^{i'}\\
\end{bmatrix}'\in\mathbb{R}^{(n+p)T_{r}}$ for $i=1,\ldots,N_{r}$. With these definitions, we have $\mathbf{H}g=\begin{bmatrix}\mathbf{Z}^{1'}& \ldots& \mathbf{Z}^{N_{r}'} \end{bmatrix}'=\mathbf{Z}_{T_{r}}^{N_{r}}$, and hence
\begin{equation} \label{touse_6}
\mathbf{Z}_{T_{r}}^{N_{r}'}\mathbf{Z}_{T_{r}}^{N_{r}}=g'\mathbf{H}'\mathbf{H}g=\tr(g'\mathbf{H}'\mathbf{H}g)=\tr(gg'\mathbf{H}'\mathbf{H}). 
\end{equation}
Taking the expectation, and from the relationship $\tr(AB)\leq \lambda_{max}(A)\tr(B)$ for real symmetric $A$ and real $B\succeq 0$ \cite{zhang2006eigenvalue},  we have
\begin{equation} \label{touse_7}
\begin{aligned}
\mathbb{E}[\tr(gg'\mathbf{H}'\mathbf{H})]&=\tr(\mathbb{E}[gg']\mathbf{H}'\mathbf{H})\leq\|\mathbb{E}[gg']\|\tr(\mathbf{H}'\mathbf{H})\\
&\leq \bar{\sigma}_{max}^2\tr(\mathbf{H}'\mathbf{H})=\bar{\sigma}_{max}^2N_{r}\tr(H'H)\\
&=\bar{\sigma}_{max}^2N_{r}(\sum_{t=0}^{T_{r}-1}\tr(\bar{G}_{t})+\sum_{k=0}^{T_{r}-1}\tr(I_{p}))\\
&=\bar{\sigma}_{max}^2N_{r}\sum_{t=0}^{T_{r}-1}(\tr(\bar{G}_{t})+p),
\end{aligned}
\end{equation}
where $\bar{G}_{t}$ is defined in \eqref{def: Gt}.
Now we consider the term $|\mathbf{Z}_{T_{r}}^{N_{r}'}\mathbf{Z}_{T_{r}}^{N_{r}}-\mathbb{E}[\mathbf{Z}_{T_{r}}^{N_{r}'}\mathbf{Z}_{T_{r}}^{N_{r}}]|$ in \eqref{touse_5}. From \eqref{touse_6}, we have
\begin{equation*}
\begin{aligned}
|\mathbf{Z}_{T_{r}}^{N_{r}'}\mathbf{Z}_{T_{r}}^{N_{r}}-\mathbb{E}[\mathbf{Z}_{T_{r}}^{N_{r}'}\mathbf{Z}_{T_{r}}^{N_{r}}]|=|g'\mathbf{H}'\mathbf{H}g-\mathbb{E}[g'\mathbf{H}'\mathbf{H}g]|.
\end{aligned}
\end{equation*}
From \cite{zhang2020concentration}, we have each component of $g$ has sub-Gaussian norm that is upper bounded by $4\bar{\sigma}_{max}$. We can apply Lemma \ref{hanson wright} to the above term with the replacement of $c_{0}$ by $\min(1,c_{0})$ to obtain
\begin{equation} \label{touse_10}
\begin{aligned}
&P(|g'\mathbf{H}'\mathbf{H}g-\mathbb{E}[g'\mathbf{H}'\mathbf{H}g]|>t)\\ &\leq 2\exp(-\min(\frac{c_{0}t^2}{256\bar{\sigma}_{max}^4\|\mathbf{H}'\mathbf{H}\|^{2}_{F}},\frac{c_{0}t}{16\bar{\sigma}_{max}^2\|\mathbf{H}'\mathbf{H}\|}))\\
&\leq 2\exp(-\min(\frac{ct^2}{\bar{\sigma}_{max}^4\|\mathbf{H}'\mathbf{H}\|^{2}_{F}},\frac{ct}{\bar{\sigma}_{max}^2\|\mathbf{H}'\mathbf{H}\|})),
\end{aligned}
\end{equation}
where $c\triangleq \frac{c_{0}}{256}$.


Fixing $\delta\in(0,\frac{2}{e})$ and setting $t=\frac{1}{c}\log(\frac{2}{\delta})\bar{\sigma}_{max}^2\tr(\mathbf{H}'\mathbf{H})$, we have

\begin{equation*}
\begin{aligned}
\frac{ct^2}{\bar{\sigma}_{max}^4\|\mathbf{H}'\mathbf{H}\|_{F}^2}&=\frac{1}{c}(\frac{\tr(\mathbf{H}'\mathbf{H})}{\|\mathbf{H}'\mathbf{H}\|_{F}})^2(\log(\frac{2}{\delta}))^2 \\
&\geq (\log(\frac{2}{\delta}))^2\geq \log(\frac{2}{\delta}),
\end{aligned}
\end{equation*} and 
\begin{equation*}
\begin{aligned}
\frac{ct}{\bar{\sigma}_{max}^2\|\mathbf{H}'\mathbf{H}\|}&=\log(\frac{2}{\delta})\frac{\tr(\mathbf{H}'\mathbf{H})}{\|\mathbf{H}'\mathbf{H}\|}\geq \log(\frac{2}{\delta}),
\end{aligned}
\end{equation*} 
where we used the fact that $\|\mathbf{H}'\mathbf{H}\|\leq\|\mathbf{H}'\mathbf{H}\|_{F}\leq \|\mathbf{H}\|^{2}_{F}= \tr(\mathbf{H}'\mathbf{H})$.

Combining the above inequalities with \eqref{touse_10}, we have with probability at least $1-\delta$
\begin{equation*}
\begin{aligned}
&|\mathbf{Z}_{T_{r}}^{N_{r}'}\mathbf{Z}_{T_{r}}^{N_{r}}-\mathbb{E}[\mathbf{Z}_{T_{r}}^{N_{r}'}\mathbf{Z}_{T_{r}}^{N_{r}}]|=|g'\mathbf{H}'\mathbf{H}g-\mathbb{E}[g'\mathbf{H}'\mathbf{H}g]|\\
&\leq \frac{1}{c}\log(\frac{2}{\delta})\bar{\sigma}_{max}^2\tr(\mathbf{H}'\mathbf{H})\\
&=\frac{1}{c}\log(\frac{2}{\delta})\bar{\sigma}_{max}^2N_{r}\sum_{t=0}^{T_{r}-1}(\tr(\bar{G}_{t})+p).
\end{aligned}
\end{equation*} 

Consequently, considering the above inequality in conjunction with \eqref{touse_7}, and from \eqref{touse_5}, we have with probability at least $1-\delta$
\begin{equation*} 
\begin{aligned}
\mathbf{Z}_{T_{r}}^{N_{r}'}\mathbf{Z}_{T_{r}}^{N_{r}}&\leq N_{r}\sum_{t=0}^{T_{r}-1}(\tr(\bar{G}_{t})+p)(\frac{1}{c}\log(\frac{2}{\delta})+1)\bar{\sigma}_{max}^2.
\end{aligned}
\end{equation*}

\end{proof}

\begin{remark} \label{c}
The constant $c$ in Lemma \ref{lemma:bound z} depends on the constant $c_{0}$ in the Hanson-Wright inequality presented in Lemma \ref{hanson wright}. More specifically, $c=\frac{\min\{1, c_0\}}{256}$, where $c_0$ is the constant from the Hanson-Wright inequality. Attempts to explicitly characterize $c_{0}$ can be found in \cite{dadush2018fast,moshksar2021absolute}. In the special case where the initial states, noises, and inputs are all Gaussian with equal variances, the paper \cite{moshksar2021absolute} shows that $c_0$ can be taken as 0.1457.  For the more general case, where all relevant stochastic signals are sub-Gaussian, we conjecture that it can be obtained by following the proof of \cite[Theorem~6.2.1]{vershynin2018high}. One can also derive similar upper bounds using the Markov inequality to get rid of the constant $c$, but at the price of having linear dependence on $\delta$ in the denominators of the bounds. We note the insights provided by Theorem \ref{data-independent bound} do not depend on the specific values of $c$. Additionally, one can also use the data-dependent bound in Theorem \ref{data-dependent bound} to obtain a computable bound for the general sub-Gaussian case.
\end{remark}

We will leverage the following definitions on $\epsilon$-net and the block martingale small-ball conditions in \cite{simchowitz2018learning}.

\begin{definition} \label{epsilon net}
Let $(T,d)$ be a metric space. Consider a subset $K\subset T$
and let $\epsilon>0$. A subset $\mathbf{N}\subseteq K$ is called an $\epsilon$-net of $K$ if every point in $K$ is within distance $\epsilon$ of some point of $\mathbf{N}$, i.e.,
\begin{equation*} 
\begin{aligned}
\forall x\in K \quad \exists x_{0}\in \mathbf{N}: d(x.x_{0})\leq \epsilon.\end{aligned}
\end{equation*}
\end{definition}

\begin{definition} \cite[Definition~2.1]{simchowitz2018learning}) \label{small ball}
Let $\{Z_{t}\}_{t\geq 1}$ be a $\{\mathcal{F}_{t}\}_{t\geq 1}$-adapted random process taking
values in $\mathbb{R}$. We say $\{Z_{t}\}_{t\geq 1}$ satisfies the $(k,v,p)$-block martingale small-ball (BMSB) condition if, for any $j \geq 0$, one has $\frac{1}{k}\sum_{i=1}^{k}P(|Z_{j+i}|\geq v |\mathcal{F}_{j})>p$ almost surely. 
\end{definition}

The following result establishes a lower bound of the smallest eigenvalue of the sample covariance matrix for general time series, leveraging the above definitions, a concentration inequality in \cite{dean2019safely}, and the ideas in \cite{matni2019tutorial}. Note that we use $v(i)$ to denote the $i$-th component of a vector $v$.

\begin{lemma} \label{PE}
Let $\{l_{t}\}_{t\geq 0}$ be a sequence of random vectors that is adapted to a filtration $\{\mathcal{F}_{t}\}_{t\geq 0}$,  where $l_{t}\in \mathbb{R}^{d}$. Let $\{\eta_{t}\}_{t\geq 1}$ be another sequence of random vectors such that $\eta_{t}$ is $\mathcal{F}_{t}$-measurable, where $\eta_{t}\in \mathbb{R}^{d}$. Further, suppose $\eta_{t+1}|\mathcal{F}_{t}$ has zero mean and independent coordinates, where each coordinate has 
bounded fourth moment for all $t\geq 0$. Suppose that $\mathbb{E}[\eta_{t+1}\eta_{t+1}'|\mathcal{F}_{t}]\succeq \sigma_{\eta}^2 I_{d}$, and $\max_{1\leq i \leq d} \frac{\mathbb{E}[\eta_{t+1}(i)^4|\mathcal{F}_{t}]}{\mathbb{E}[\eta_{t+1}(i)^2|\mathcal{F}_{t}]^2}\leq c_{\eta}$ for all $t\geq 0$, where $\sigma_{\eta},c_{\eta}\in \mathbb{R}_{>0}$.  Let $\zeta=\frac{\sigma_{\eta}}{c_{1}c_{\eta}}$, where $c_{1}$ is a positive absolute constant. Define the sequence $z_{t+1}=l_{t}+\eta_{t+1}$ for $t\geq 0$, where $z_{0} \in \mathbb{R}^{d}$. Fix $\delta>0$ and a constant $M\geq \frac{\zeta^{2} (T-1)}{16}$ such that it holds $\|\sum_{t=0}^{T-1}z_{t}z_{t}'\| \leq M$ with probability at least $1-\frac{\delta}{2}$. Then, if $T\geq 8c_{1}^2c_{\eta}^2(\log\frac{2}{\delta}+d\log\frac{144M}{\zeta^2(T-1)})+1$, we have with probability at least $1-\delta$,
\begin{equation*}
\sum_{t=0}^{T-1}z_{t}z_{t}' \succeq \frac{\zeta^2 (T-1)}{32}I_{d}.
\end{equation*}
\end{lemma}

\begin{proof}
Note that for any fixed $v\in \mathcal{S}^{d-1}$ and $t\geq 0$, we have
\begin{equation*}
v'z_{t+1}|\mathcal{F}_{t}=v'l_{t}|\mathcal{F}_{t}+v'\eta_{t+1}|\mathcal{F}_{t}.
\end{equation*}
From Proposition \ref{anti}, we have
\begin{equation*}
P(|v'z_{t+1}|\geq \sqrt{\frac{\sigma_{\eta}^2}{2}}|\mathcal{F}_{t})\geq\frac{1}{c_{1}\times c_{\eta}}
\end{equation*}
almost surely. Since the scalar process $\{v'z_{t}\}_{t\geq 1}$ is adapted to the filtration $\{\mathcal{F}_{t}\}_{t\geq 1}$, the above inequality implies that $v'z_{t}$ satisfies the $(1,\frac{\sqrt{2}\sigma_{\eta}}{2},\frac{1}{c_{1} \times c_{\eta}})$ BMSB condition (see Definition \ref{small ball}). 
Denoting $m=\frac{\zeta^2(T-1)}{16}$, we can now apply Lemma \ref{small ball PE} to obtain
\begin{equation}\label{fixed v}
\begin{aligned}
P(\sum_{t=1}^{T-1} (v' z_{t})^2\leq m)&\leq \exp(\frac{-(T-1)}{8c_{1}^2c_{\eta}^2})\\
&\leq \exp(\log \frac{\delta}{2}+\log(\frac{\zeta^2 (T-1)}{144M})^d) \\
&=\frac{\delta}{2} (\frac{m}{9M})^d,
\end{aligned}
\end{equation}
where the last inequality is due to our assumption that $T\geq 8c_{1}^2c_{\eta}^2(\log\frac{2}{\delta}+d\log\frac{144M}{\zeta^2(T-1)})+1$.

Since $M\geq m$, we can let $\mathbf{N}(\frac{m}{4M})$ be a $\frac{m}{4M}$- net of $\mathcal{S}^{d-1}$ with the smallest cardinality (see Definition \ref{epsilon net}). From Lemma \ref{covering number}, we know that there are at most $(\frac{9M}{m})^d$ elements in $\mathbf{N}(\frac{m}{4M})$. Applying a union bound to combine the events in \eqref{fixed v} for all $v \in \mathbf{N}(\frac{m}{4M})$, we have with probability at least $1-\frac{\delta}{2}$
\begin{equation} \label{net event}
\begin{aligned}
\min_{v\in \mathbf{N}(\frac{m}{4M})}\sum_{t=1}^{T} (v' z_{t})^2&\geq m.
\end{aligned}
\end{equation}
Note that for any realization of the matrix $\sum_{t=0}^{T-1}z_{t}z_{t}'$, we can fix a $v^*\in \mathcal{S}^{d-1}$ such that $\lambda_{min}(\sum_{t=0}^{T-1}z_{t}z_{t}') = \sum_{t=0}^{T-1}v^{*'}z_{t}z_{t}'v^{*}$, and let $v_{0}\in \mathbf{N}(\frac{m}{4M})$ be a vector such that $\|v^*-v_{0}\|\leq \frac{m}{4M}$, to obtain
\begin{equation*}
\begin{aligned}
&\lambda_{min}(\sum_{t=0}^{T-1}z_{t}z_{t}') =\sum_{t=0}^{T-1}(v_{0}+v^{*}-v_{0})'z_{t}z_{t}'(v_{0}+v^{*}-v_{0})\\
&=\sum_{t=0}^{T-1}v_{0}'z_{t}z_{t}'v_{0}+\sum_{t=0}^{T-1}v_{0}'z_{t}z_{t}'(v^{*}-v_{0})\\
&\quad\quad\quad +\sum_{t=0}^{T-1}(v^{*}-v_{0})'z_{t}z_{t}'v_{0}+\sum_{t=0}^{T-1}(v^{*}-v_{0})'z_{t}z_{t}'(v^{*}-v_{0})\\
&\geq \sum_{t=0}^{T-1}(v_{0}'z_{t})^2-\frac{m}{2M}\|\sum_{t=0}^{T-1}z_{t}z_{t}\|\\
&\geq \sum_{t=1}^{T-1}(v_{0}'z_{t})^2-\frac{m}{2M}\|\sum_{t=0}^{T-1}z_{t}z_{t}\|\\
&\geq \min_{v\in \mathbf{N}(\frac{m}{4M})} \sum_{t=1}^{T-1}(v'z_{t})^2-\frac{m}{2M}\|\sum_{t=0}^{T-1}z_{t}z_{t}\|.\\
\end{aligned}
\end{equation*}
Applying a union bound to combine the event $\|\sum_{t=0}^{T-1}z_{t}z_{t}'\| \leq M$ and the event in \eqref{net event}, we have with probability at least $1-\delta$,
\begin{equation*}
\begin{aligned}
&\lambda_{min}(\sum_{t=0}^{T-1}z_{t}z_{t}') \geq \frac{m}{2}= \frac{\zeta^2(T-1)}{32}.\\
\end{aligned}
\end{equation*}
\end{proof}

We have the following result.
\begin{lemma}
Let Assumption \ref{assumption} hold. Denote $\bar{\zeta}=\frac{\bar{\sigma}_{min}}{c_{1}\bar{\sigma}_{*}}$ and $\hat{\zeta}=\frac{\hat{\sigma}_{min}}{c_{1}\hat{\sigma}_{*}}$, where $c_{1}$ is a positive absolute constant. Fixing $\delta\in(0,1)$, suppose that $N_{r}T_{r}\geq 8c_{1}^2\bar{\sigma}_{*}^2(\log\frac{2}{\delta}+(n+p)\log\frac{144\bar{g}(\frac{\delta}{2})}{\bar{\zeta}^2(N_{r}T_{r}-1)})+1$, $N_{p}T_{p}\geq 8c_{1}^2\hat{\sigma}_{*}^2(\log\frac{2}{\delta}+(n+p)\log\frac{144\hat{g}(\frac{\delta}{2})}{\hat{\zeta}^2(N_{p}T_{p}-1)})+1$, $\bar{g}(\frac{\delta}{2}) \geq \frac{\bar{\zeta}^2(N_{r}T_{r}-1)}{16}$, and $\hat{g}(\frac{\delta}{2}) \geq \frac{\hat{\zeta}^2(N_{p}T_{p}-1)}{16}$,  where $\bar{g}(\frac{\delta}{2}), \hat{g}(\frac{\delta}{2})$ are defined in Lemma \ref{lemma:bound z}.  Then, we have with probability at least $(1-\delta)^2$
\begin{equation*}
ZQZ'\succeq \frac{(N_{r}T_{r}-1)\bar{\zeta}^2+q(N_{p}T_{p}-1)\hat{\zeta}^2}{32} I_{n+p}.
\end{equation*}
\label{prop:PE}
\end{lemma}
\begin{proof}
We have 
\begin{equation}
\begin{aligned} \label{touse_3}
ZQZ'&=\bar{Z}\bar{Z}'+q\hat{Z}\hat{Z}'=\sum_{i=1}^{N_{r}}\sum_{t=0}^{T_{r}-1}\bar{z}^{i}_{t}\bar{z}^{i'}_{t}+q\sum_{j=1}^{N_{p}}\sum_{k=0}^{T_{p}-1}\hat{z}^{j}_{k}\hat{z}^{j'}_{k}.\\
\end{aligned}
\end{equation}

We now focus on the first summation in \eqref{touse_3} since the analysis for the second one is essentially the same. We can define the sequence $\{z_{t}\}_{t\geq0}$ as
\begin{equation*}
z_{t}=\begin{cases}
  \bar{z}^{1}_{t}  & \text{if} \quad 0\leq t\leq T_{r}-1 \\
  \bar{z}^{2}_{t-T_{r}} &  \text{if} \quad T_{r}\leq t\leq 2T_{r}-1 \\
   \bar{z}^{3}_{t-2T_{r}} &  \text{if} \quad 2T_{r}\leq t\leq 3T_{r}-1\\
  \vdots&  \vdots ,\\
\end{cases}
\end{equation*}
 where $\bar{z}^{a}_{b}$ for $a>N_{r}$ and $b=0,\ldots, T_{r}-1$ are generated using the same way as $\bar{z}^{a}_{b}$ for $a=1,\ldots,N_{r}$ and $b=0,\ldots, T_{r}-1$. In words, $z_{t}$ is the sequence formed by concatenating the sequence $\{\bar{z}_{t}^{1}\}_{t=0}^{T_{r}-1}, \{\bar{z}_{t}^{2}\}_{t=0}^{T_{r}-1}, \ldots$. The sequences $\{w_{t}\}_{t\geq0}$ and $\{u_{t}\}_{t\geq0}$ are defined similarly using the signals $\bar{w}^{i}_{t}$ and $\bar{u}^{i}_{t}$. Further, we define the sequence $\{x_{t}\}_{t\geq 0}$ as $x_{t}=\bar{x}^{t+1}_{0}$, where $\bar{x}^{a}_{0}$ for $a>N_{r}$ are generated using the same way as $\bar{x}^{a}_{0}$ for $a=1,\ldots,N_{r}$. With these definitions, we have

\begin{equation*}
\begin{aligned}
&\sum_{i=1}^{N_{r}}\sum_{t=0}^{T_{r}-1}\bar{z}_{t}^{i}\bar{z}^{i'}_{t}=\sum_{t=0}^{N_{r}T_{r}-1}z_{t}z_{t}'.
\end{aligned}
\end{equation*}

We now verify the conditions in Lemma \ref{PE}. Note that for $t$ satisfying $t\geq 0$ and $(t+1)\bmod T_{r}\neq 0$, we have
\begin{equation*}
\begin{aligned}
z_{t+1}=
\begin{bmatrix}
\Theta z_{t}\\
0\end{bmatrix}
+
\begin{bmatrix}
w_{t}\\
u_{t+1}\end{bmatrix}. 
\end{aligned}
\end{equation*}
For  $t$ satisfying $t\geq 0$ and $(t+1)\bmod T_{r}= 0$, we have 
\begin{equation*}
\begin{aligned}
z_{t+1}=
\begin{bmatrix}
x_{\frac{t+1}{T_{r}}}\\
u_{t+1}\end{bmatrix}.
\end{aligned}
\end{equation*}

Let $l_{t}=\begin{bmatrix}
\Theta z_{t}\\
0\end{bmatrix}$ and $\eta_{t+1}=\begin{bmatrix}
w_{t}\\
u_{t+1}\end{bmatrix}$ for $t$ satisfying $t\geq 0$ and $(t+1)\bmod T_{r}\neq 0$. Similarly, let
$l_{t}=\mathbf{0}$ and $\eta_{t+1}=\begin{bmatrix}
x_{\frac{t+1}{T_{r}}}\\
u_{t+1}\end{bmatrix}$ for $t$ satisfying $t\geq 0$ and $(t+1)\bmod T_{r}= 0$.  Define the filtration $\{\mathcal{F}_{t}\}_{t\geq 0}$, where $\mathcal{F}_{t}=\sigma(\{l_{i}\}_{i=0}^{t}\cup\{\eta_{j}\}_{j=1}^{t})$. We have $l_{t}$ is $\mathcal{F}_{t}$-measurable for $t\geq 0$, $\eta_{t}$ is $\mathcal{F}_{t}$-measurable for $t\geq 1$, $\mathbb{E}[\eta_{t+1}\eta_{t+1}'|\mathcal{F}_{t}]\succeq \bar{\sigma}_{min}^2I_{n+p}$ and $\max_{1\leq i \leq (n+p)} \frac{\mathbb{E}[\eta_{t+1}(i)^4|\mathcal{F}_{t}]}{\mathbb{E}[\eta_{t+1}(i)^2|\mathcal{F}_{t}]^2}\leq \bar{\sigma}_{*}$ for $t\geq 0$ due to our assumption.

Fixing $\delta\in(0,1)$, from Lemma \ref{lemma:bound z}, we have $\|\sum_{t=0}^{N_{r}T_{r}-1}z_{t}z_{t}'\|\leq \bar{g}(\frac{\delta}{2})$ with probability at least $1-\frac{\delta}{2}$. Consequently, letting $N_{r}T_{r}\geq 8c_{1}^2\bar{\sigma}_{*}^2(\log\frac{2}{\delta}+(n+p)\log\frac{144\bar{g}(\frac{\delta}{2})}{\bar{\zeta}^2(N_{r}T_{r}-1)})+1$, we can apply Lemma \ref{PE} to get
\begin{equation*}
\begin{aligned}
\sum_{i=1}^{N_{r}}\sum_{t=0}^{T_{r}-1}\bar{z}_{t}^{i}\bar{z}^{i'}_{t}=\sum_{t=0}^{N_{r}T_{r}-1}z_{t}z_{t}'&\succeq \frac{(N_{r}T_{r}-1)\bar{\zeta}^2}{32}I_{n+p}
\end{aligned}
\end{equation*}
with probability at least $1-\delta$. Applying a similar procedure for the second summation in \eqref{touse_3} and leveraging the independence of data, we have with probability at least $(1-\delta)^2$
\begin{equation*}
\begin{aligned}
&\sum_{i=1}^{N_{r}}\sum_{t=0}^{T_{r}-1}\bar{z}^{i}_{t}\bar{z}^{i'}_{t}+q\sum_{j=1}^{N_{p}}\sum_{k=0}^{T_{p}-1}\hat{z}^{j}_{k}\hat{z}^{j'}_{k}\\
&\succeq \frac{(N_{r}T_{r}-1)\bar{\zeta}^2+q(N_{p}T_{p}-1)\hat{\zeta}^2}{32} I_{n+p}.
\end{aligned}
\end{equation*}
\end{proof}


We will use the following lemma, which provides an upper bound for self-normalized martingales.
\begin{lemma}(\cite[Theorem~1]{abbasi2011improved} )
\label{martingale_bound}
Let $\{\mathcal{F}_{t}\}_{t\geq 0}$ be a filtration. Let $\{{w}_{t}\}_{t\geq 1}$ be a real valued stochastic process such that $w_{t}$ is $\mathcal{F}_{t}$-measurable, and $w_{t}$ is conditionally sub-Gaussian on $\mathcal{F}_{t-1}$ with parameter $R^2$. Let $\{z_{t}\}_{t\geq 1}$ be an $\mathbb{R}^{m}$-valued stochastic process such that $z_{t}$
is $\mathcal{F}_{t-1}$-measurable. Assume that $V$ is a $m\times m$ dimensional positive definite matrix. For all $t\geq 0$, define
\begin{equation*}
\begin{aligned}
&\bar{V}_{t}=V+\sum_{s=1}^{t}z_{s}z_{s}', S_{t}=\sum_{s=1}^{t}w_{s}z_{s}.\\
\end{aligned}
\end{equation*}
Then, for any $\delta>0$, and for all $t\geq0$,
\begin{equation*}
\begin{aligned}
&P(\|\bar{V}_{t}^{-\frac{1}{2}}S_{t}\|^2\leq2R^2\log(\frac{\det(\bar{V}_{t}^{\frac{1}{2}})\det(V^{-\frac{1}{2}})}{\delta}))\geq 1-\delta.\\
\end{aligned}
\end{equation*}
\end{lemma}

The following lemma generalizes the above result to the case where $w_{t}$ is multi-dimensional, and will be used to bound the error term $\|WQZ'(ZQZ')^{-\frac{1}{2}}\|$. The proof is similar to \cite[Proposition~V.4]{matni2019tutorial}.
\begin{lemma} \label{martingale_bound_multi}
Let $\{\mathcal{F}_{t}\}_{t\geq 0}$ be a filtration. Let $\{{w}_{t}\}_{t\geq 1}$ be a  $\mathbb{R}^{n}$-valued stochastic process such that $w_{t}$ is $\mathcal{F}_{t}$-measurable, and $w_{t}$ is conditionally sub-Gaussian on $\mathcal{F}_{t-1}$ with parameter $R^2$. Let $\{z_{t}\}_{t\geq 1}$ be a $\mathbb{R}^{m}$-valued stochastic process such that $z_{t}$
is $\mathcal{F}_{t-1}$-measurable. Assume that $V$ is a $m\times m$ dimensional positive definite matrix. For all $t\geq 0$, define
\begin{equation*}
\begin{aligned}
&\bar{V}_{t}=V+\sum_{s=1}^{t}z_{s}z_{s}', S_{t}=\sum_{s=1}^{t}z_{s}w_{s}'.\\
\end{aligned}
\end{equation*}
Then, for any $\delta>0$, and for all $t\geq0$,
\begin{equation*}
\begin{aligned}
&P(\|\bar{V}_{t}^{-\frac{1}{2}}S_{t}\|\leq\sqrt{\frac{32}{9}R^{2}(\log\frac{9^n}{\delta}+\frac{1}{2}\log\det(\bar{V}_{t}V^{-1})})\\
&\geq 1-\delta.\\
\end{aligned}
\end{equation*}
\end{lemma}
\begin{proof}
We have
\begin{equation*}
\begin{aligned}
&\|\bar{V}_{t}^{-\frac{1}{2}}S_{t}\|= \|\bar{V}_{t}^{-\frac{1}{2}}\sum_{s=1}^{t}z_{s}w_{s}'\|=\sup_{v\in \mathcal{S}^{n-1}}\|\bar{V}_{t}^{-\frac{1}{2}}\sum_{s=1}^{t}z_{s}w_{s}'v\|.\\
\end{aligned}
\end{equation*}
Note that for any fixed unit vector $v\in \mathcal{S}^{n-1}$, the random variable $w_{s}^{'}v$ is conditionally sub-Gaussian with parameter $R^{2}$. 
Let $\mathbf{N}(\frac{1}{4})$ be a $\frac{1}{4}$- net of $\mathcal{S}^{n-1}$ with the smallest cardinality (see Definition \ref{epsilon net}). From Lemma \ref{covering number}, we know that there are at most $9^n$ elements in $\mathbf{N}(\frac{1}{4})$. For any fixed $\delta\in(0,1)$ and $v\in \mathbf{N}(\frac{1}{4})$, we can apply Lemma \ref{martingale_bound} to obtain with probability at least $1-\frac{\delta}{9^n}$
\begin{equation*}
\begin{aligned}
\|\bar{V}_{t}^{-\frac{1}{2}}\sum_{s=1}^{t}z_{s}w_{s}'v\|^{2}&\leq2R^{2}\log\frac{9^n\det(\bar{V}_{t}^{\frac{1}{2}})\det(V^{-\frac{1}{2}})}{\delta}\\
&=2R^{2}(\log\frac{9^n}{\delta}+\frac{1}{2}\log\det(\bar{V}_{t}V^{-1})).
\end{aligned}
\end{equation*}

Applying a union bound over all $9^n$ events, from Lemma \ref{cover bound}, we have with probability at least $1-\delta$
\begin{equation*}
\begin{aligned}
\|\bar{V}_{t}^{-\frac{1}{2}}\sum_{s=1}^{t}z_{s}w_{s}'\|&\leq\frac{4}{3}\sup_{v\in \mathbf{N}(\frac{1}{4})}\|\bar{V}_{t}^{-\frac{1}{2}}\sum_{s=1}^{t}z_{s}w_{s}'v\|\\
&\leq\sqrt{\frac{32}{9}R^{2}(\log\frac{9^n}{\delta}+\frac{1}{2}\log\det(\bar{V}_{t}V^{-1})}.
\end{aligned}
\end{equation*}

\end{proof}

\subsection{Proof of Theorem \ref{data-independent bound}}

\begin{proof}
Recall that the system identification error in \eqref{error_W} (using $\lambda=0$) satisfies
\begin{equation}
\begin{aligned}
\|\Theta_{WLS}-\Theta\|&\leq \|(ZQZ')^{-\frac{1}{2}}ZQW'\|\|(ZQZ')^{-\frac{1}{2}}\|\\
&+\|\Delta QZ'\|\|(ZQZ')^{-1}\|,
\end{aligned}
\end{equation}
under the invertibility assumption. Let $N_{r}T_{r}, N_{p}T_{p}, \delta$ satisfy the conditions in Theorem \ref{data-independent bound}, and let $V=\frac{N_{r}T_{r}\bar{\zeta}^2+qN_{p}T_{p}\hat{\zeta}^2}{33} I_{n+p}$. From Lemma \ref{prop:PE}, we have with probability at least $1-2\delta$
\begin{equation} 
ZQZ'\succeq V, \label{event1}
\end{equation} 
conditioning on which we have
\begin{equation} \label{event9}
\|(ZQZ')^{-\frac{1}{2}}\|\leq\|V^{-\frac{1}{2}}\|,
\end{equation}
where we used the relationship $\frac{(N_{r}T_{r}-1)\bar{\zeta}^2+q(N_{p}T_{p}-1)\hat{\zeta}^2}{32}\geq \frac{N_{r}T_{r}\bar{\zeta}^2+qN_{p}T_{p}\hat{\zeta}^2}{33}$ when $\min\{N_{r}T_{r}, N_{p}T_{p}\}\geq 33$ in \eqref{event1}, and Lemma \ref{inverse bound} in conjunction with Lemma \ref{root bound} for \eqref{event9}.

Further, conditioning on \eqref{event1}, we also have  $ZQZ' \succeq V \Rightarrow  2ZQZ' \succeq ZQZ'+V\Rightarrow (ZQZ')^{-1} \preceq 2 (ZQZ'+V)^{-1}$, where we used Lemma \ref{inverse bound}. Applying Lemma \ref{norm bound}, we have
\begin{equation}  \label{event4}
\|(ZQZ')^{-\frac{1}{2}}ZQW'\|\leq\sqrt{2}\|(ZQZ'+V)^{-\frac{1}{2}}ZQW'\|.
\end{equation} 
Next, to use Lemma \ref{lemma:bound z}, we define a new pair of sequences $\{z_{t}\}_{t\geq1}$ and $\{w_{t}\}_{t\geq1}$ using the signals used in the terms $ZQZ'=\sum_{i=1}^{N_{r}}\sum_{t=0}^{T_{r}-1}\bar{z}^{i}_{t}\bar{z}^{i'}_{t}+q\sum_{j=1}^{N_{p}}\sum_{k=0}^{T_{p}-1}\hat{z}^{j}_{k}\hat{z}^{j'}_{k}$ and $ZQW'=\sum_{i=1}^{N_{r}}\sum_{t=0}^{T_{r}-1}\bar{z}^{i}_{t}\bar{w}^{i'}_{t}+q\sum_{j=1}^{N_{p}}\sum_{k=0}^{T_{p}-1}\hat{z}_{k}^{j}\hat{w}^{j'}_{k}$. That is,
\begin{equation*}
\medmath{z_{t}=\begin{cases}
  \bar{z}^{1}_{t-1}  & \text{if} \quad 1\leq t\leq T_{r} \\
  \bar{z}^{2}_{t-T_{r}-1} &  \text{if} \quad T_{r}+1\leq t\leq 2T_{r} \\
   \bar{z}^{3}_{t-2T_{r}-1} &  \text{if} \quad 2T_{r}+1\leq t\leq 3T_{r}\\
  \vdots&  \vdots \\
 \bar{z}^{N_{r}}_{t-(N_{r}-1)T_{r}-1} &  \text{if} \quad (N_{r}-1)T_{r}+1\leq t\leq N_{r}T_{r}\\
  \sqrt{q}\hat{z}^{1}_{t-N_{r}T_{r}-1} &  \text{if} \quad N_{r}T_{r}+1\leq t\leq N_{r}T_{r}+T_{p}\\
   \sqrt{q}\hat{z}^{2}_{t-N_{r}T_{r}-T_{p}-1} &  \text{if} \quad N_{r}T_{r}+T_{p}+1\leq t\leq N_{r}T_{r}+2T_{p} \\
   \sqrt{q}\hat{z}^{3}_{t-N_{r}T_{r}-2T_{p}-1} &  \text{if} \quad N_{r}T_{r}+2T_{p}+1\leq t\leq N_{r}T_{r}+3T_{p}\\
   \vdots&  \vdots ,\\
\end{cases}}
\end{equation*}
and 
\begin{equation*}
\medmath{w_{t}=\begin{cases}
  \bar{w}^{1}_{t-1}  & \text{if} \quad 1\leq t\leq T_{r} \\
  \bar{w}^{2}_{t-T_{r}-1} &  \text{if} \quad T_{r}+1\leq t\leq 2T_{r} \\
   \bar{w}^{3}_{t-2T_{r}-1} &  \text{if} \quad 2T_{r}+1\leq t\leq 3T_{r}\\
  \vdots&  \vdots \\
 \bar{w}^{N_{r}}_{t-(N_{r}-1)T_{r}-1} &  \text{if} \quad (N_{r}-1)T_{r}+1\leq t\leq N_{r}T_{r}\\
  \sqrt{q}\hat{w}^{1}_{t-N_{r}T_{r}-1} &  \text{if} \quad N_{r}T_{r}+1\leq t\leq N_{r}T_{r}+T_{p}\\
   \sqrt{q}\hat{w}^{2}_{t-N_{r}T_{r}-T_{p}-1} &  \text{if} \quad N_{r}T_{r}+T_{p}+1\leq t\leq N_{r}T_{r}+2T_{p} \\
   \sqrt{q}\hat{w}^{3}_{t-N_{r}T_{r}-2T_{p}-1} &  \text{if} \quad N_{r}T_{r}+2T_{p}+1\leq t\leq N_{r}T_{r}+3T_{p}\\
   \vdots&  \vdots ,\\
\end{cases}}
\end{equation*}
where $\sqrt{q}\hat{z}^{a}_{b}, \sqrt{q}\hat{w}^{a}_{b}$ for $a>N_{p}$ and $b=0,\ldots, T_{p}-1$ are generated using the same way as $\sqrt{q}\hat{z}^{a}_{b}, \sqrt{q}\hat{w}^{a}_{b}$ for $a=1,\ldots,N_{p}$ and $b=0,\ldots, T_{p}-1$.

Consequently, we have 
\begin{equation*}
\begin{aligned}
ZQZ'=\sum_{t=1}^{N_{r}T_{r}+N_{p}T_{p}}z_{t}z_{t}'
\end{aligned},
\end{equation*} and
\begin{equation*}
\begin{aligned}
ZQW'=\sum_{t=1}^{N_{r}T_{r}+N_{p}T_{p}}z_{t}w_{t}'.
\end{aligned}
\end{equation*}
Now define the filtration $\{\mathcal{F}_{t}\}_{t\geq 0}$, where $\mathcal{F}_{t}=\sigma(\{z_{i+1}\}_{i=0}^{t}\cup\{w_{j}\}_{j=1}^{t})$.
With these definitions, we can see that the noise terms $w_{t}$ are $\mathcal{F}_{t}$-measurable, and $w_{t}|\mathcal{F}_{t-1}$ are sub-Gaussian with parameter $\max(\sigma_{\bar{w}}^2, q\sigma_{\hat{w}}^2)$ for all $t\geq1$. Consequently, we can apply Lemma \ref{martingale_bound_multi} to obtain with probability at least $1-\delta$
\begin{equation} \label{event2}
\begin{aligned}
&\sqrt{2}\|(ZQZ'+V)^{-\frac{1}{2}}ZQW'\|\\
&\leq 3\max(\sigma_{\bar{w}}, \sqrt{q}\sigma_{\hat{w}})\sqrt{\log\frac{9^{n}}{\delta}+\log\det((ZQZ'+V)V^{-1})}.
\end{aligned}
\end{equation}
Further, from Lemma \ref{lemma:bound z}, we have with probability at least $1-2\delta$
\begin{equation} \label{event3}
\begin{aligned}
&\det((ZQZ'+V)V^{-1})\leq \frac{\|ZQZ'+V\|^{n+p}}{\det(V)}\leq\\
& \frac{(\|\sum_{i=1}^{N_{r}}\sum_{t=0}^{T_{r}-1}\bar{z}^{i}_{t}\bar{z}^{i'}_{t}\|+q\|\sum_{i=1}^{N_{p}}\sum_{t=0}^{T_{p}-1}\hat{z}_{t}^{i}\hat{z}^{i'}_{t}\|+\|V\|)^{n+p}}{\det(V)}\\
&\leq (\frac{33(\bar{g}(\delta)+q\hat{g}(\delta))}{N_{r}T_{r}\bar{\zeta}^2+qN_{p}T_{p}\hat{\zeta}^2}+1)^{n+p}=\phi^{n+p}.
\end{aligned}
\end{equation}
Applying a union bound over the events in \eqref{event4}, \eqref{event2}, and \eqref{event3}, we have with probability at least $1-5\delta$
\begin{equation} \label{NO2}
\begin{aligned}
&\|(ZQZ')^{-\frac{1}{2}}ZQW'\|\leq\\
& 3\max(\sigma_{\bar{w}}, \sqrt{q}\sigma_{\hat{w}})\sqrt{\log\frac{9^{n}}{\delta}+(n+p)\log(\phi)}.
\end{aligned}
\end{equation}

Next, conditioning on the event in Lemma \ref{lemma:bound z}, notice that we also have 
\begin{equation} \label{NO3}
\begin{aligned}
\|\Delta QZ'\|&=\|\sum_{i=1}^{N_{p}}\sum_{t=0}^{T_{p}-1}q\delta_{\Theta}\hat{z}_{t}^{i}\hat{z}_{t}^{i'}\|\leq q\|\delta_{\Theta}\|\|\sum_{i=1}^{N_{p}}\sum_{t=0}^{T_{p}-1}\hat{z}^{i}_{t}\hat{z}^{i'}_{t}\|\\
&\leq q\|\delta_{\Theta}\| \hat{g}(\delta).
\end{aligned}
\end{equation}

Finally, combining \eqref{event9}, \eqref{NO2} and \eqref{NO3}, we have the desired result.
\end{proof}

\subsection{Proof of Corollary \ref{When to use}}
\begin{proof}
Setting $q=0$, from Theorem \ref{data-independent bound}, we have with probability at least $1-\delta$
\begin{equation} \label{bound 0}
\begin{aligned}
&\|\Theta_{WLS}-\Theta\|\\
&\leq\frac{20\sigma_{\bar{w}}\sqrt{\log\frac{9^{n}}{\delta}+(n+p)\log(\frac{33\bar{g}(\delta)}{N_{r}T_{r}\bar{\zeta}^2}+1)}}{\sqrt{N_{r}T_{r}\bar{\zeta}^2}}.
\end{aligned}
\end{equation}

When $q\neq 0$, from Theorem \ref{data-independent bound}, after some algebraic manipulations, we can show that with probability at least $1-\delta$
\begin{equation} \label{bound not 0}
\begin{aligned}
&\|\Theta_{WLS}-\Theta\|\\
&\leq \frac{20\max(\frac{\sigma_{\bar{w}}}{\sqrt{q}}, \sigma_{\hat{w}})\sqrt{\log\frac{9^{n}}{\delta}+(n+p)\log(\frac{\gamma}{\zeta^2}+1)}}{\sqrt{N_{p}T_{p}\hat{\zeta}^2}}\\
&+\|\delta_{\Theta}\|\frac{\gamma}{\hat{\zeta}^2}.
\end{aligned}
\end{equation}
The proof follows by setting the upper bound in \eqref{bound 0} to be greater than the one in \eqref{bound not 0}.
\end{proof}

\subsection{Proof of Proposition \ref{data-independent bound refined}}

\begin{proof}
The proof follows the same procedure as the proof of Theorem \ref{data-independent bound}. However,  instead of bounding the term $\|\Delta QZ'(ZQZ')^{-1}\|$ by $\|\Delta QZ'\|\|(ZQZ')^{-1}\|$, we use the following bound
\begin{equation} \label{touse}
\begin{aligned}
&\|\Delta QZ'(ZQZ')^{-1}\|=\\
&\|(\sum_{j=1}^{N_{p}}\sum_{k=0}^{T_{p}-1}q\delta_{\Theta}\hat{z}_{k}^{j}\hat{z}_{k}^{j'})(\sum_{i=1}^{N_{r}}\sum_{t=0}^{T_{r}-1}\bar{z}^{i}_{t}\bar{z}^{i'}_{t}+q\sum_{j=1}^{N_{p}}\sum_{k=0}^{T_{p}-1}\hat{z}^{j}_{k}\hat{z}^{j'}_{k})^{-1}\|\\
&\leq \|\delta_{\Theta}\|\|I_{n+p}-\\
&\quad\quad (\sum_{i=1}^{N_{r}}\sum_{t=0}^{T_{r}-1}\bar{z}^{i}_{t}\bar{z}^{i'}_{t})(\sum_{i=1}^{N_{r}}\sum_{t=0}^{T_{r}-1}\bar{z}^{i}_{t}\bar{z}^{i'}_{t}+q\sum_{j=1}^{N_{p}}\sum_{k=0}^{T_{p}-1}\hat{z}^{j}_{k}\hat{z}^{j'}_{k})^{-1})\|\\
&\leq \|\delta_{\Theta}\|(1+ \|\sum_{i=1}^{N_{r}}\sum_{t=0}^{T_{r}-1}\bar{z}^{i}_{t}\bar{z}^{i'}_{t}\|\\
&\quad\quad\quad\quad\quad\quad \times\|(\sum_{i=1}^{N_{r}}\sum_{t=0}^{T_{r}-1}\bar{z}^{i}_{t}\bar{z}^{i'}_{t}+q\sum_{j=1}^{N_{p}}\sum_{k=0}^{T_{p}-1}\hat{z}^{j}_{k}\hat{z}^{j'}_{k})^{-1})\|),
\end{aligned}
\end{equation}
where we used the relationship $AB=(C+A)B-CB$ for real matrices $A,B,C$ in the first inequality. When the events in Lemma \ref{prop:PE} and Lemma \ref{lemma:bound z} happen, we can upper bound the right hand side of the last inequality in \eqref{touse} by
$\|\delta_{\Theta}\|(1+\frac{33 \bar{g}(\delta)  }{N_{r}T_{r}\bar{\zeta}^2+qN_{p}T_{p}\hat{\zeta}^2})$.

\end{proof}

\subsection{Proof of Theorem \ref{data-dependent bound}}
\begin{proof}
From \eqref{error_W}, note that the system identification error satisfies
\begin{equation}
\begin{aligned}
&\|\Theta_{WLS}-\Theta\|\\
&\leq \lambda\|\Theta\|\|(ZQZ'+\lambda I_{n+p})^{-1}\|\\
&+\|(ZQZ'+\lambda I_{n+p})^{-\frac{1}{2}}ZQW'\|\|(ZQZ'+\lambda I_{n+p})^{-\frac{1}{2}}\|\\
&+\|\Delta QZ'(ZQZ'+\lambda I_{n+p})^{-1}\|\\
&\leq \frac{\lambda\|\Theta\|}{\lambda_{min}(ZQZ'+\lambda I_{n+p})}\\
&+\frac{\|(ZQZ'+\lambda I_{n+p})^{-\frac{1}{2}}ZQW'\|}{\sqrt{\lambda_{min}(ZQZ'+\lambda I_{n+p})}}\\
&+q\|\delta_{\Theta}\|\|\hat{Z}\hat{Z}'(ZQZ'+\lambda I_{n+p})^{-1}\|.
\end{aligned}
\end{equation}

Note that all terms in the above inequality can be evaluated from data, except for the term $\|(ZQZ'+\lambda I_{n+p})^{-\frac{1}{2}}ZQW'\|$. 
We can follow a similar procedure to apply Lemma \ref{martingale_bound_multi} as in the proof of Theorem \ref{data-independent bound}. Fixing $\delta>0$, from Lemma \ref{martingale_bound_multi}, we have with probability at least $1-\delta$
\begin{equation*} 
\begin{aligned}
&\|(ZQZ'+\lambda I_{n+p})^{-\frac{1}{2}}ZQW'\|\\
&\leq \max(\sigma_{\bar{w}}, \sqrt{q}\sigma_{\hat{w}})\sqrt{\frac{32}{9}(\log\frac{9^{n}}{\delta}+\frac{1}{2}\log\det(\bar{V}))}.
\end{aligned}
\end{equation*}
The result then follows.
\end{proof}

\subsection{Auxiliary Results}

\begin{lemma}(\cite[Corollary~4.2.13]{vershynin2018high}) \label{covering number}
Let $\epsilon>0$, and let $\mathbf{N}(\mathcal{S}^{n-1},\epsilon)$ be the smallest possible cardinality of an  $\epsilon$-net of the unit Euclidean sphere $\mathcal{S}^{n-1}$. We have the following inequality:
\begin{equation*} 
\begin{aligned}
\mathbf{N}(\mathcal{S}^{n-1},\epsilon) \leq (\frac{2}{\epsilon}+1)^n.
\end{aligned}
\end{equation*}
\end{lemma}

\begin{lemma} (\cite[Lemma~4.4.1]{vershynin2018high}) \label{cover bound}
Let $A$ be an $m$ by $n$ matrix and $\epsilon \in [0,1)$. Then, for any $\epsilon$-net $\mathbf{N}$ of the sphere $\mathcal{S}^{n-1}$, we have
\begin{equation*}
\begin{aligned}
\|A\|\leq \frac{1}{1-\epsilon}\cdot \sup_{x\in \mathbf{N}}\|Ax\|.
\end{aligned}
\end{equation*}
\end{lemma}

\begin{lemma} (\cite[Lemma~3]{chan1985hermitian}) \label{inverse bound}
Let $A\in \mathbb{R}^{n\times n}$ and $B\in \mathbb{R}^{n\times n}$ be positive definite matrices. If $A \preceq B$, then we have $A^{-1}\succeq B^{-1}$.
\end{lemma}

\begin{lemma} (\cite[Theorem~2]{chan1985hermitian}) \label{root bound}
Let $A\in \mathbb{R}^{n\times n}$ and $B\in \mathbb{R}^{n\times n}$ be positive semidefinite matrices. If $A \preceq B$, then we have $A^{\frac{1}{2}}\preceq B^{\frac{1}{2}}$.
\end{lemma}

\begin{lemma} \label{norm bound}
Let $A\in \mathbb{R}^{n\times n}$ and $B\in \mathbb{R}^{n\times n}$ be positive semidefinite matrices. Let $C \in \mathbb{R}^{n\times m}$. If $A \preceq B$, then we have
\begin{equation*}
\begin{aligned}
\|A^{\frac{1}{2}}C\|\leq \|B^{\frac{1}{2}}C\|.
\end{aligned}
\end{equation*}
\end{lemma}
\begin{proof}
From  $A \preceq B$, we have
\begin{equation*} 
\begin{aligned}
C'AC\preceq C'BC ,
\end{aligned}
\end{equation*}
which implies 
\begin{equation*} 
\begin{aligned}
\|A^{\frac{1}{2}}C\|=\sqrt{\lambda_{max}(C'AC)}\leq \sqrt{\lambda_{max}(C'BC)}=\|B^{\frac{1}{2}}C\|.
\end{aligned}
\end{equation*}
\end{proof}

\begin{proposition}  \label{prop:system matrices bound}
Assuming that $\rho(\bar{A})<1$ and $\rho(\hat{A})<1$, we have both $\tr(\bar{G}_{t})$ and $\tr(\hat{G}_{t})$ are $\mathcal{O}(1)$. If $\rho(\bar{A})=1$ and $\rho(\hat{A})=1$, we have $\tr(\bar{G}_{t})=\mathcal{O}(t^{2\bar{\kappa}})$ and $\tr(\hat{G}_{t})=\mathcal{O}(t^{2\hat{\kappa}})$, where $\bar{\kappa}$, $\hat{\kappa}$ are the largest
Jordan blocks corresponding to the unit eigenvalues of $\bar{A}$ and $\hat{A}$, respectively.
\end{proposition}
\begin{proof}
We only consider the term $\tr(\bar{G}_{t})$ as the term $\tr(\hat{G}_{t})$ is essentially the same. Defining $\bar{F_{t}}=\begin{bmatrix}
I_{n}&\bar{A}&\cdots&\bar{A}^t& \bar{B}&\bar{A}\bar{B}&\bar{A}^2\bar{B}&\cdots&\bar{A}^{t-1}\bar{B}
\end{bmatrix}  \in \mathbb{R}^{n\times (tn+tp+n)}$ for $t\geq 0$, we have $\bar{G_{t}}=\bar{F_{t}}\bar{F_{t}}'$. Further, we have
\begin{equation*} 
\begin{aligned}
\tr(\bar{G}_{t})&\leq n\|\bar{G}_{t}\|\leq n\|\bar{F}_{t}\|^2\leq n(\sum_{i=0}^{t}\|\bar{A}^{i}\|+\sum_{i=0}^{t-1}\|\bar{A}^{i}\|\|\bar{B}\|)^2.
\end{aligned}
\end{equation*}
From \cite[Lemma~E.2.]{tsiamis2019finite}, we have $\sum_{i=0}^{t}\|\bar{A}^{i}\|$ is $\mathcal{O}(1)$ for strictly stable systems, and $\mathcal{O}(t^{\hat{\kappa}})$ for marginally stable systems. The result then follows.
\end{proof}

\begin{lemma} \cite[Proposition~2.5]{simchowitz2018learning}) \label{small ball PE}
Suppose that $\{Z_{t}\}_{t\geq 1}$ satisfies the $(k,v,p)$-BMSB condition. Denoting $\floor{\cdot}$ as the floor function, we have
\begin{equation*} 
\begin{aligned}
P(\sum_{i=1}^{T}Z_{i}^2 \leq \frac{v^2p^2}{8}k\floor{T/k})\leq \exp(-\frac{\floor{T/k}p^2}{8}).
\end{aligned}
\end{equation*}
\end{lemma}

\begin{proposition} \cite[Proposition~C.2.]{dean2019safely} \label{anti}
Let $\mu \in \mathbb{R}$ and $M \in \mathbb{R}^{d \times d}$ be a full rank matrix. Let $w \in \mathbb{R}^{d}$ be a random vector such that each coordinate $w(i)$ has positive variance and finite fourth moment. Further, each coordinate $w(i)$ is independent and zero-mean. Then, for any fixed $v\in \mathcal{S}^{d-1}$,
\begin{equation*} 
\begin{aligned}
P_{w}(| \langle v , \; \mu+Mw \rangle |\geq \sqrt{\lambda_{min}(M\Sigma M')/2})\geq \frac{1}{c_{1}C_{w}},
\end{aligned}
\end{equation*}
where $\Sigma=\mathbb{E}_{w}[ww']$, $c_{1}$ is an absolute constant,     and $C_{w}=\max_{1\leq i \leq d} \frac{\mathbb{E}[w(i)^4]}{\mathbb{E}[w(i)^2]^2}$.
\end{proposition}

\begin{remark} \label{c1}
The constant $c_{1}$ is due to the application of the Rosenthal’s inequality. As suggested in \cite[Proposition~C.2.]{dean2019safely}, one can take $c_{1}=192$.
\end{remark}


\bibliographystyle{IEEEtran}
\bibliography{main}

\end{document}